\documentclass{article}

\usepackage{PRIMEarxiv}

\usepackage[utf8]{inputenc} 
\usepackage[T1]{fontenc}    
\usepackage[numbers]{natbib} 
\usepackage[hidelinks]{hyperref}       
\usepackage{url}            
\usepackage{booktabs}       
\usepackage{amsfonts}       
\usepackage{nicefrac}       
\usepackage{microtype}      
\usepackage{xcolor}
\usepackage{fancyvrb}
\usepackage{fvextra}
\usepackage{tcolorbox}
\tcbuselibrary{breakable,skins}
\newcommand{\toolong}[1]{}
\DefineVerbatimEnvironment{PromptVerbatim}{Verbatim}{fontsize=\small}
\newtcolorbox{promptbox}[1][]{
  colback=gray!12,
  colframe=gray!45,
  boxrule=0.4pt,
  arc=2pt,
  left=5pt,
  right=5pt,
  top=4pt,
  bottom=4pt,
  #1,
}
\usepackage{algorithm}
\usepackage{algorithmic}
\usepackage{amsmath,amssymb,amsfonts,mathtools,bm, amsthm}
\usepackage{thmtools}
\usepackage{wrapfig}
\newtheorem{theorem}{Theorem}
\newtheorem{lemma}{Lemma}
\newtheorem{claim}{Claim}

\usepackage[nameinlink,capitalise,noabbrev]{cleveref}
\Crefname{assumption}{Assumption}{Assumptions}
\crefname{assumption}{assumption}{assumptions}
\usepackage{pgfplots}
\usepackage{subcaption}
\usepackage{multirow}
\pgfplotsset{compat=1.18}
\usepgfplotslibrary{fillbetween}

\newcommand{\vect}[1]{\ensuremath{\bm{#1}}}
\newcommand{\E}{\ensuremath{\mathbb{E}}}

\newcommand{\parenthese}[1]{\left(#1\right)}
\newcommand{\bracket}[1]{\left[#1\right]}

\newcommand{\curlybracket}[1]{\left\{#1\right\}}
\newcommand{\xcal}{\ensuremath{\mathcal{X}}}

\DeclareMathOperator*{\argmax}{arg\,max}

\usepackage{ifthen}
\newboolean{showcomments}
\setboolean{showcomments}{true}

\newboolean{showfinished}
\setboolean{showfinished}{true}

\ifthenelse{\boolean{showcomments}}
{ \newcommand{\mynote}[3]{
		\fbox{\bfseries\sffamily\scriptsize#1}
		{\small$\blacktriangleright$\textsf{\emph{\color{#3}{#2}}}$\blacktriangleleft$}}
	\newcommand{\zzz}[1]{{\setlength{\fboxsep}{2pt}\fcolorbox{black}{yellow}{\textsf{\emph{#1}}}}\xspace}}
{ \newcommand{\mynote}[3]{}
	\newcommand{\zzz}[1]{}}

\title{Bandits in Prod: Hyperparameter Optimization at Inference Time
}

\author{
 Louis Abraham \\
  Tiime\\
  Paris, France \\
  \texttt{louis.abraham@tiime.fr} \\
   \And
 Tuan-Anh Nguyen \\
  Tiime\\
  Paris, France \\
  \texttt{tuan-anh.nguyen@tiime.fr} \\
  \And
 Nicolas Devatine \\
  Tiime\\
  Paris, France \\
  \texttt{nicolas.devatine@tiime.fr} \\
}

\begin{document}
\maketitle

\begin{abstract}
  Many production systems can assess a configuration only by using it on live requests and observing noisy feedback. Modern agentic systems are a prominent example, with inference-time choices such as model selection, retrieval depth, prompting strategy, and decoding temperature, yet often with no representative validation data. We formalize this setting as Online Hyperparameter Optimization (OHPO) and cast it as an infinitely many-armed bandit over mixed and conditional search spaces. We introduce IMABO, a general framework that combines any bandit policy for choosing among already sampled configurations with any oracle for proposing new ones. We instantiate it with IMOSS, a restart-free anytime policy whose active set grows as $t^{\beta}$, and prove an expected cumulative quantile-regret bound of $O(p_\rho^{-1/\beta} + T^{(1+\beta)/2})$, where $\beta\in(0,1)$ controls active-set growth and $p_\rho$ lower-bounds the probability that a proposed configuration falls in the top-$\rho$ fraction of the search space. 
  We combine IMOSS with three practical oracles: a Tree-structured Parzen Estimator, an incumbent-mutation oracle driven by a per-coordinate bandit, and a pretrained tabular foundation model, all three improving over the uniform random oracle baseline.
  IMABO outperforms all baselines in terms of regret across diverse OHPO settings, from tuning classical machine-learning models to configuring LLM-based agents. Our implementation is available at \url{https://github.com/Tiime-Software/IMABO}.
\end{abstract}
\section{Introduction}

Hyperparameter optimization (HPO) is usually performed offline by evaluating configurations on a fixed dataset and keeping the best one, with methods such as grid search, random search, or Bayesian optimization~\cite{bergstra2012random,snoek2012practical,li2017hyperband}. A growing class of systems cannot be tuned this way. Applications built around large language models served by external providers expose only inference-time settings, such as the choice of model, the decoding temperature, or the prompting strategy. Without a dataset that reflects real usage, a configuration can be evaluated only by serving it and observing how it performs. We call this problem online hyperparameter optimization (OHPO). Because every evaluation is a served request, exploration is never free: what matters is the reward accumulated over the whole request stream, not the quality of a final recommendation, and the horizon of that stream is not known in advance. OHPO has received little attention, falling between (offline) HPO, which evaluates configurations before deployment, and online learning, which rarely exploits the structure of a hyperparameter space.

This problem has the structure of a multi-armed bandit (MAB), where each configuration is an arm, and each served request is a pull of the chosen arm. But a hyperparameter space is typically continuous or mixed, combining real-valued, integer, and categorical settings, so the possible configurations far outnumber the requests the system will ever serve. Bandit problems in which the arms outnumber the available pulls are known as infinitely many-armed bandits (IMAB). Existing approaches either assume a metric under which nearby configurations earn similar rewards, an assumption that is meaningless for categorical parameters such as the choice of model, or draw new configurations at random, learning nothing from past rewards.

We address this problem with IMABO (Infinitely Many-Armed Bandit with Oracle), a framework that pairs a bandit policy over a growing active set of configurations with an oracle that proposes which configuration to admit next. Our contributions are as follows.
\begin{itemize}
    \item We introduce and formalize the OHPO problem, cast it as an infinitely many-armed bandit, and propose IMABO, a framework that combines any bandit policy with any oracle, requiring only that a configuration can be scored from live feedback and making no further assumptions about the search space.
    \item We introduce IMOSS, an anytime bandit policy for online hyperparameter optimization, prove a sublinear cumulative quantile-regret bound, and instantiate it with several arm-creation oracles: IMOSS-Random, IMOSS-TPE, IMOSS-TabPFN, and IMOSS-mutate-KL$\times$PE. These variants natively support continuous, integer, and categorical hyperparameters while remaining computationally efficient for production. Our implementation is publicly available.\footnote{Code and experiments: \url{https://github.com/Tiime-Software/IMABO}}
    
    \item We validate our method on both continuous and discrete HPOBench tasks, and on a real LLM-based question-answering system over a mixed search space of retrieval, prompting, and model choices.
\end{itemize}

\section{Related Work}
\label{sec:related}

Our work lies at the intersection of HPO and MAB theory, and relates most closely to recent methods that bring bandit ideas into production.

\paragraph{Hyperparameter optimization.} 

Most HPO methods run offline, optimizing a fixed objective over a budget. Grid and random search remain strong baselines on problems of low effective dimension~\cite{bergstra2012random}. Model-based methods fit a surrogate to choose where to sample, through Gaussian-process Bayesian optimization~\cite{snoek2012practical,shahriari2016taking}, random forests for conditional and categorical spaces~\cite{hutter2011smac}, or the Tree-structured Parzen Estimator (TPE)~\cite{bergstra2011algorithms}, which needs no metric and handles both continuous and categorical parameters, and whose density ratio has since been recast as a classification problem~\cite{tiao2021bore,song2022lfbo}. Mixed continuous-categorical spaces have also been addressed by Gaussian-process optimization restricted to trust regions~\cite{wan2021casmopolitan}, which still assumes a kernel over the space. A~recent class of surrogates avoids this per-task fitting by using a tabular foundation model, a transformer pretrained once on synthetic data that predicts how a configuration will perform in a single forward pass, with no retraining during the search. Such surrogates have been used both to forecast the final accuracy of a partially trained model~\cite{rakotoarison2024incontext} and to predict the objective directly in high-dimensional problems~\cite{yu2025gitbo}.

\paragraph{Multi-, continuum-, and infinite-armed bandits.}
In the stochastic $K$-armed bandit setting, index policies such as UCB1~\cite{auer2002finite} and the minimax-optimal MOSS~\cite{audibert2009minimax} pull the arm with the highest upper confidence bound.
The anytime MOSS index~\cite{degenne2016anytime} keeps the minimax guarantee without knowing the horizon $T$ in advance. When the number of arms far exceeds the horizon, sublinear regret against the global optimum requires an additional assumption~\cite{berry1997bandit}. One class of methods assumes that nearby arms are similar, either via an explicit metric (continuum-armed~\cite{kleinberg2004nearly}) or via metric-free optimistic partitioning~\cite{bubeck2011xarmed,munos2011optimistic,valko2013stochastic,bartlett2019simple}, several of which we use as baselines. Such a metric is natural for a continuous parameter but not for a categorical choice. A second class draws each arm from a fixed distribution over configurations (a reservoir) and measures regret relative to it~\cite{carpentier2015simple}. \citet{chaudhuri2018quantile} take the target to be a quantile of the mean reward induced by the reservoir, their QRM1 and QRM2 running MOSS on a growing pool but staying anytime by restarting on a doubling schedule, while UCB-AIR~\cite{wang2008algorithms} grows its arm set without restarts. No existing method is at once restart-free, anytime, and able to bound cumulative quantile regret over a growing arm set. A~parallel line, contextual bandits, replaces per-arm means by a reward model shared across arms, linear~\cite{li2010contextual}, neural~\cite{zhou2020neuralucb}, tree-based~\cite{nilsson2024tree}, or, concurrently with this work, a prior-data fitted network~\cite{tan2026pfnts}; all score a fixed enumerated action set and bound regret against its best arm, whereas our oracle uses such a surrogate to \emph{propose} arms not yet in the active set, a step \citet{tan2026pfnts} leave open for large action spaces.

\paragraph{Bandits \& HPO.} Bandit ideas first entered HPO in the offline training setting, where Successive Halving~\cite{jamieson2016nonstochastic}, Hyperband~\cite{li2017hyperband}, and BOHB~\cite{falkner2018bohb} allocate a limited training budget across configurations by casting the choice as best-arm identification, and where CoCaBO~\cite{ru2020cocabo} runs a bandit over the categorical coordinates alongside Gaussian-process optimization over the continuous ones. \citet{shang2019dttts} extend this best-arm view to a fixed reservoir of configurations with D-TTTS, but evaluate on offline validation data and target simple regret. Bandits have long been proposed as a successor to A/B testing~\cite{scott2010modern,chapelle2011empirical}. A recent line of work tunes hyperparameters online, bounding dynamic regret over continuous coordinates under a Lipschitz assumption~\cite{lu2022nonstationary,kang2023online} or running \textsc{Exp3} over mixed actions that must be enumerated in advance and scored on a validation split~\cite{liu2023online}. Population-based training~\cite{jaderberg2017population} and its bandit variants~\cite{parkerholder2021pb2mix,wan2022bgpbt} instead tune \emph{during} a training run by inheriting weights across parallel workers, which a single serving stream against a fixed provider-side model cannot supply. A separate line couples bandit selection with mutation of the best solution found so far: bandit-based random mutation hill-climbing~\cite{liu2016banditrmhc} mutates it one component at a time, using a set of per-coordinate bandits to decide which component to change, and Mutant-UCB~\cite{bregere2024mutantucb} injects evolutionary mutation operators into a UCB best-arm-identification bandit for model selection.

\paragraph{Inference-time tuning of LLM systems.} LLM systems have made our setting widespread, and the closest prior work tunes them online with bandits. Only two of these methods take a \emph{configuration} as an arm. 
AutoRAG-HP~\cite{fu2024autorag} tunes a RAG pipeline from live rewards alone with a hierarchy of per-hyperparameter bandits, so cost scales with the sum rather than the product of the axis sizes, but every axis must be discretized in advance, and no regret guarantee is given. Each trial changes one hyperparameter and leaves the others so the search is a random walk on a pre-declared grid that credits each observed reward to the single coordinate it just moved; we include it in our benchmark.
BanditSpec~\cite{hou2025banditspec} does bound a stopping-time regret over speculative-decoding configurations, but its rounds are token steps inside a single generation and its objective is decoding latency rather than answer quality. In the rest of this literature, as in routing benchmarks~\cite{hu2024routerbench}, an arm is a bare model identity drawn from an enumerated pool, with no structure to interpolate and nothing to propose outside it. Within this group, \citet{xia2024tiucb} commit to a model before observing one reward, with a logarithmic bound, but over a handful of candidates and under an increasing-then-converging reward trend. The others either relax the protocol, by cascading several models per request~\cite{chen2023frugalgpt,aggarwal2024automix} or choosing after seeing the query~\cite{ong2024routellm}, or depend on offline artifacts we do without, such as preference data, per-model quality and cost predictors, or acceptance thresholds. Across all of these, the candidate set is declared in advance, and most optimize cost or latency rather than reward quality.

\section{Online Hyperparameter Optimization}
\label{sec:problem}

Let $\xcal$ be a hyperparameter search space mixing continuous, integer, and categorical coordinates (a temperature, a retrieval depth, a choice of model). It may be tree-structured, with some parameters becoming active only once others are fixed, as when a chosen model determines which hyperparameters apply~\cite{bergstra2011algorithms}. The horizon is unknown. At each round $t = 1, 2, \dots$, a learner commits a configuration $\vect{x}_t \in \xcal$ to serve the next request, then observes a reward $r_t \in [0,1]$ (a task success, a click, or a graded answer) drawn independently from an unknown distribution $\mathcal{P}_{\vect{x}_t}$ with mean $\mu_{\vect{x}_t}$. This reward is the only feedback: there is no validation set to query on the side, and past rounds cannot be replayed under a different configuration. 

Competing with the global optimum $\mu^\star = \sup_{\vect{x}\in\xcal}\mu_{\vect{x}}$ is not possible in this setting, as $\xcal$ contains far more configurations than the system will ever try. We therefore measure performance against a quantile of the mean reward, defined relative to a baseline sampling distribution $P_0$ over $\xcal$. Such a baseline is easy to construct even on a tree-structured space, by drawing each decision uniformly, or log-uniformly when a parameter is specified on a logarithmic scale. For $\rho \in (0,1)$, let $\mu_\rho$ be the top-$\rho$ quantile of the mean reward $\mu_{\vect{x}}$ when $\vect{x}$ is drawn from $P_0$, and let $\mathrm{TOP}_\rho = \{\vect{x} \in \xcal \mid \mu_{\vect{x}} \geq \mu_\rho\}$ be the corresponding set of good configurations. A draw from $P_0$ lands in $\mathrm{TOP}_\rho$ with probability at least $\rho$, so about $1/\rho$ draws are enough to find a good configuration. We evaluate a learner by its cumulative $\rho$-regret $\mathcal{R}^{\rho,+}_T = \sum_{t=1}^{T} \parenthese{\mu_\rho - \mu_{\vect{x}_t}}_+$, with $a_+ = \max(a,0)$. It accumulates the reward lost on every request served by a configuration below the target. We also define the simple $\rho$-regret $\mathcal{R}^{\rho,+}_{(T)} = \parenthese{\mu_\rho - \mu_{\vect{x}_{(T)}}}_+$, which measures the quality of the configuration $\vect{x}_{(T)}$ returned after $T$ rounds. Optimizing a running system in this setting is the problem we call online hyperparameter optimization (OHPO).
\section{Infinitely Many-Armed Bandit with Oracle}
\label{sec:imabo}

In our formulation of OHPO, each configuration is an arm and each served request is a pull. A hyperparameter space $\xcal$ is often continuous or mixed, giving infinitely many arms, or at least far more than the system can ever try, which places the problem in the infinitely many-armed bandit (IMAB) setting studied by \citet{berry1997bandit} and \citet{wang2008algorithms}.\footnote{The formulation applies equally to a finite $\xcal$, as long as $\xcal$ is large enough.}

An IMAB algorithm maintains an active set $\mathcal{M}_t \subseteq \xcal$ of the arms seen so far, starting from the empty set and adding one arm at a time. At each round it makes a binary choice. It either pulls an arm already in $\mathcal{M}_t$, selected by whatever bandit rule it runs over the active set, or admits a new arm drawn from $\xcal$, adds it to $\mathcal{M}_t$, and pulls it. The standard formulation draws each new arm from a fixed reservoir distribution over $\xcal$, so the arm admitted at any round is statistically independent of every reward seen before it.

IMABO changes only the expansion step. When a new arm enters $\mathcal{M}_t$, it is supplied not by a fixed reservoir distribution but by an oracle $\mathcal{O}$, a mapping from the history of past pulls and their rewards to a configuration in $\xcal$. Setting $\mathcal{O}$ to draw from the baseline $P_0$ regardless of the history recovers the standard IMAB algorithm with reservoir $P_0$, whereas an oracle that reads the history can steer new arms toward the regions of $\xcal$ that past rewards suggest are promising. We place no further restriction on $\mathcal{O}$, which could be a density estimator over $\xcal$, or even a large language model prompted with past configurations and their rewards. Algorithm~\ref{alg:imabo} gives IMABO for an arbitrary oracle $\mathcal{O}$ and bandit policy $\mathcal{A}$, and a concrete algorithm is obtained by choosing these two components. Selecting which arm of $\mathcal{M}_t$ to serve is itself a standard bandit problem with its own exploration--exploitation trade-off. A policy that restarts its statistics at intervals may either reset the oracle along with its own statistics or let the oracle keep the full history, a choice left to the instantiation.

\begin{algorithm}[t]
    \caption{IMABO ($\mathcal{A}$, $\mathcal{O}$)}
    \label{alg:imabo}
    \begin{algorithmic}[1]
    \STATE \textbf{Input:} space $\xcal$, oracle $\mathcal{O}$, bandit policy $\mathcal{A}$
    \STATE $\mathcal{M} \gets \emptyset$
    \FOR{$t = 1, 2, 3, \ldots$}
        \IF{$\mathcal{A}.\textsc{Expand}()$}
            \STATE $\vect{x} \gets \mathcal{O}.\textsc{Suggest}()$ \COMMENT{draw a new arm}
            \STATE $\mathcal{M} \gets \mathcal{M} \cup \{\vect{x}\}$
        \ELSE
            \STATE $\vect{x} \gets \mathcal{A}.\textsc{Select}()$ \COMMENT{pull a known arm}
        \ENDIF
        \STATE Serve $\vect{x}$ on the next request, observe reward $r$, update its statistics
    \ENDFOR
    \end{algorithmic}
\end{algorithm}

\subsection{The IMOSS Family}
\label{sec:imabo1}

In this paper, we use IMOSS as the bandit policy. IMOSS is an anytime bandit policy adapted from MOSS~\cite{audibert2009minimax}; pairing it with four different oracles yields the four concrete algorithms we study. IMOSS-Random draws new configurations from the baseline $P_0$, while the other three use past rewards to target promising regions of $\xcal$. IMOSS-TPE proposes from a density ratio fitted on the whole active set; IMOSS-mutate-KL$\times$PE instead improves the best arm found so far one coordinate at a time, using a one-dimensional Parzen estimator on the selected coordinate to propose local mutations; and IMOSS-TabPFN ranks its candidates with a tabular foundation model. \Cref{sec:oracle-ablation} compares the four and asks which to use on which problem.

\subsubsection{IMOSS}
\label{subsec:imoss}

We adapt MOSS~\cite{audibert2009minimax,degenne2016anytime} to the IMAB setting. Standard MOSS scores each arm from its observed rewards and serves the highest scorer, attaining the minimax-optimal cumulative regret $O(\sqrt{KT})$ for a $K$-armed bandit, but its score requires the horizon $T$ and a fixed arm count $K$, and neither is available here. IMOSS therefore uses the anytime MOSS index~\cite{degenne2016anytime}, which substitutes the current round $t$ for $T$; IMOSS recomputes this index at every round with the current active-set size $K_t=|\mathcal{M}_t|$ in place of $K$. It serves the arm of highest index
\begin{equation}
    \mathcal{B}_t(\vect{x})
    = \hat{\mu}_{\vect{x}}
    + \sqrt{
        \frac{1+\alpha}{2}\cdot\frac{\max\parenthese{0,\ \log\parenthese{\tfrac{t}{K_t\, n_{\vect{x}}}}}}{n_{\vect{x}}}
    }
    \label{eq:moss}
\end{equation}
where $\hat{\mu}_{\vect{x}}$ is the empirical mean reward of arm $\vect{x}$, $n_{\vect{x}}$ is the number of times it has been served, and $\alpha>0$ is a confidence parameter carried over from the anytime MOSS index~\cite{degenne2016anytime}. The square-root term is an exploration bonus that shrinks as $n_{\vect{x}}$ grows, and also shrinks as $K_t$ grows: admitting a new arm spreads the same budget over more candidates, so it lowers the bonus of every arm already in $\mathcal{M}_t$. The parameter $\alpha$ scales this bonus through the factor $(1+\alpha)/2$, a larger value pushing the policy toward exploration. In the experiments, we set $\alpha = 0.1$ as suggested by \citet{degenne2016anytime}.

Growth of the active set follows a power law similar to UCB-AIR~\cite{wang2008algorithms}. Specifically, a new arm is requested from the oracle whenever $|\mathcal{M}_t| < t^{\beta}$, with $\beta \in (0,1)$ fixed, yielding an active-set size of $\lfloor t^{\beta} \rfloor$. The exponent trades off discovery against estimation, as a set that grows too slowly delays finding a good region of $\xcal$, while one that grows too fast leaves each arm too few observations to be ranked reliably, and \Cref{thm:quantile-regret} quantifies this balance as a regret bound. In our experiments, we set $\beta=0.5$, so that the active set grows as $\sqrt{t}$, which we found empirically to strike this balance. A brief warm-up precedes the first round, drawing $N_s = 10$ configurations from $P_0$ by default and serving each once before the index takes over. Algorithm~\ref{alg:moss} states the resulting policy.

\toolong{IMOSS matches the best-known quantile-regret rate of QRM2~\cite{chaudhuri2018quantile}, which restarts its bandit on a doubling schedule, while remaining restart-free and anytime. Nothing here is specific to MOSS, and any minimax-optimal anytime index, such as KL-UCB or OCUCB~\cite{garivier2011klucb,lattimore2016ocucb}, could be paired with the same growth schedule.}


\begin{algorithm}[t]
    \caption{IMOSS ($\beta$, $\alpha$, $N_s$, proposal distribution $P_{\mathcal{O}}$)}
    \label{alg:moss}
    \begin{algorithmic}[1]
    \STATE \textbf{Input:} active set $\mathcal{M}_t$, round $t$, exponent $\beta \in (0,1)$, confidence $\alpha > 0$, warm-up size $N_s$
    \IF{$|\mathcal{M}_t| < N_s$}
        \STATE draw $\vect{x} \sim P_0$ \COMMENT{warm-up}
        \STATE $\mathcal{M}_t \gets \mathcal{M}_t \cup \{\vect{x}\}$
        \STATE \textbf{return} $\vect{x}$
    \ELSIF{$|\mathcal{M}_t| < t^{\beta}$}
        \STATE draw $\vect{x} \sim P_{\mathcal{O}}$ \COMMENT{admit a new arm}
        \STATE $\mathcal{M}_t \gets \mathcal{M}_t \cup \{\vect{x}\}$
        \STATE \textbf{return} $\vect{x}$
    \ELSE
        \STATE \textbf{return} $\displaystyle\argmax_{\vect{x} \in \mathcal{M}_t}\ \mathcal{B}_t(\vect{x})$ \COMMENT{serve a known arm}
    \ENDIF
    \end{algorithmic}
\end{algorithm}

\subsubsection{Theoretical Analysis}
\label{subsec:theoretical-analysis}

We bound the cumulative $\rho$-regret of IMOSS under a single top-$\rho$ coverage condition on the oracle: whenever a new arm is requested, the oracle proposes, given the history so far, a top-$\rho$ configuration with probability at least $p_\rho$. The analysis places no other restriction on the oracle, which enters the bound only through $p_\rho$; since $\mathrm{TOP}_\rho$ is defined relative to the baseline $P_0$, so is $p_\rho$. The oracle of IMOSS-Random samples from that same baseline, hence satisfies the condition with $p_\rho = \rho$ by construction, and a more sophisticated oracle is meant to raise this discovery probability. To bring such an oracle within the scope of the analysis, it suffices to mix it with the baseline in the spirit of $\varepsilon$-greedy: drawing the new arm from $P_0$ with probability $\varepsilon$ and querying the oracle otherwise guarantees $p_\rho \ge \varepsilon\rho$.

\begin{restatable}{theorem}{quantileregret}
    \label{thm:quantile-regret}
    Let $\rho \in (0,1)$. We assume that rewards are supported on $[0,1]$ and that each proposed configuration lies in $\mathrm{TOP}_\rho$ with conditional probability at least $p_\rho > 0$:
\begin{equation*}
\forall t \ge 1,~\mathbb P_{P_{\mathcal{O}}}(\vect{x}_t \in \mathrm{TOP}_\rho \mid \mathcal{F}_{t-1}) \ge p_\rho.
\end{equation*}
Then the expected cumulative $\rho$-regret of IMOSS satisfies
    \begin{align*}
        \mathbb E\bracket{\mathcal R_T^{\rho,+}}
        =
        O\!\parenthese{
            p_\rho^{-1/\beta}
            +
            T^{(1+\beta)/2}
        }.
    \end{align*}
\end{restatable}


We provide detailed proofs for the regret bound in \Cref{thm:quantile-regret} and analyze the memory and running time of IMOSS, along with the per-call complexity of each oracle, in \Cref{sec:theory-analysis,sec:complexity}, respectively.

\subsubsection{IMOSS-TPE}
\label{subsec:tpe}

As a first oracle, we propose the Tree-structured Parzen Estimator (TPE)~\cite{bergstra2011algorithms}, which splits the arms of $\mathcal{M}_t$ at a score quantile $\gamma$ into a good group and a complementary bad group. It fits a Parzen density to each ($\ell$ and $g$, one kernel per parameter), and selects as the proposal, from a pool sampled from $\ell$, the candidate with the highest ratio $\ell(\vect{x})/g(\vect{x})$, a proxy for expected improvement. TPE requires no metric on $\xcal$ and handles continuous, integer, and categorical parameters natively, which suits our search spaces.
Standard TPE is defined for offline HPO, where each configuration is evaluated once and the split acts directly on the observed function values. In our setting, an arm instead accumulates noisy rewards over repeated pulls, so the split requires a per-arm score, and we order arms by the MOSS index $\mathcal{B}_t$ of Equation~\eqref{eq:moss} rather than by the empirical mean $\hat{\mu}_{\vect{x}}$, since an arm that appears strong after only a few pulls should not yet be classified as reliably good, and the index folds this uncertainty into the classification. This follows the optimistic principle used by index policies and by tree searches such as HOO~\cite{bubeck2011xarmed}. In our implementation, we set the good fraction to $\gamma = 0.3$, so the good group contains the top $30\%$ of arms, and each call draws $24$ candidates from $\ell$, matching Optuna's default TPE configuration. The same oracle applies unchanged when $\xcal$ is finite but too large to enumerate.\footnote{On a finite space the oracle may propose a configuration already in $\mathcal{M}_t$. IMOSS then serves that arm again rather than admit a duplicate; this deviation from the theory avoids requerying the oracle until it yields a fresh arm.}

\subsubsection[IMOSS-mutate-KL x PE]{IMOSS-mutate-KL$\times$PE}
\label{subsec:klxtpe}

TPE fits its density over a whole configuration, with one kernel per parameter,
so it samples every coordinate independently and can combine values from
different good arms into a proposal that matches no single one of them. Our
second oracle works the other way: it is \emph{local} and uses no global
surrogate model. It takes the best arm found so far $\hat{\vect{x}}$ and changes
exactly one of its parameters, in the spirit of bandit-driven local
search~\cite{liu2016banditrmhc,bregere2024mutantucb}. We split each such proposal into three steps:
\begin{itemize}
    \item \emph{(a) Which arm to change:} the arm of $\mathcal{M}_t$ with
    the highest empirical mean, $\hat{\vect{x}} = \argmax_{\vect{x} \in \mathcal{M}_t} \hat{\mu}_{\vect{x}}$. This is deliberately not
    the arm the allocation rule is pulling right now: MOSS ranks arms by their
    mean plus an exploration bonus, and mutating that optimistic pick instead of
    the best arm measurably hurts the oracle.
    \item \emph{(b) Which parameter to change:} a KL-UCB bandit~\cite{garivier2011klucb}
    over the $d$ coordinates. Choosing coordinate $i$ means ``mutate parameter
    $i$'', and its reward is the empirical mean of the arm that mutation
    produced. Each proposed arm contributes one such reward, and that reward is
    updated as later pulls refine the arm's mean rather than added again. We use
    the KL index instead of a Hoeffding bonus because the rewards this bandit
    sees sit near the top of $[0,1]$: there a Hoeffding confidence width is
    several times wider than the gaps it needs to tell apart, and the bandit
    would fall back to trying every parameter in turn.
    \item \emph{(c) Which value to give it:} a one-dimensional Parzen estimator
    (PE) on that coordinate alone. We split $\mathcal{M}_t$ into a good and a bad
    group exactly as in \Cref{subsec:tpe} and fit the univariate pair $\ell/g$ to the $i$-th coordinate values of the arms in each group. The oracle then draws $24$ candidate values
    from $\ell$ and keeps the one with the highest ratio $\ell/g$ that differs
    from the current value of $\hat{\vect{x}}$, so the mutation always changes
    something.
\end{itemize}
Until $10$ arms have each returned at least one reward, the oracle samples from
$P_0$ instead, since fitting a density and forming a good/bad split needs a
populated active set. On a finite space a mutation often lands on an arm already
in $\mathcal{M}_t$; that arm is then pulled again. The oracle fits a single
one-dimensional density per call and trains no surrogate model, so its per-call
cost is negligible next to the model-based oracle below.

\subsubsection{IMOSS-TabPFN}
\label{subsec:tabfm}

We instantiate the oracle with
TabPFN-3~\cite{hollmann2025tabpfn}, a tabular foundation model pretrained to
predict a target from a table of examples with no further training required.
We use the regression variant of the publicly released
model.\footnote{\url{https://huggingface.co/Prior-Labs/tabpfn_3}} At each
expansion step, the oracle builds a table from the arms tried so far, with
parameters as columns and mean rewards as the
label,\footnote{Providing all the pulls instead would give TabPFN the raw
observations and let arms with more evaluations carry more weight, yet in
our experiments the two choices give similar results, so we keep the more
compact per-arm table labeled by the mean reward.} and provides it to
TabPFN as in-context examples. The pool of $100$ candidates is built mostly \emph{around the best arm so far} $\hat{\vect{x}}$ rather than sampled from $P_0$. $90$ of the candidates are mutations of that arm: one of its parameters is picked at random and given a new value, drawn by a
Gaussian step whose width is $0.1$ of that parameter's range (taken in log
space for log-scaled parameters, and clamped to stay in range); for a
categorical parameter the new value is instead one of its other categories,
chosen uniformly. The remaining $10$ candidates are drawn from $P_0$, so the
pool always keeps some reach into unexplored regions. Drawing the whole pool
from $P_0$ instead puts no candidate near a good arm and measurably weakens
the oracle.

Before scoring, we clean up the pool: identical candidates are merged, and any
candidate that is already an open arm is dropped. This matters because the pool
is meant to pick the next \emph{new} arm to admit. On a finite space, without
this step, TabPFN would often score an already-open neighbor of $\hat{\vect{x}}$
above every new candidate; the oracle would then keep re-proposing arms it
already has, and the active set would stop growing and fall behind its
$t^{\beta}$ schedule.

Each surviving candidate is then scored by TabPFN. For a candidate $\vect{x}$,
we read the $q=0.975$ quantile $\hat{F}_{\vect{x}}^{-1}(q)$ of the predictive
distribution TabPFN returns, and propose the candidate with the largest such
quantile. This is an upper-confidence-bound rule in quantile
form~\cite{yu2025gitbo}: if the predictive distribution were Gaussian, ranking
by this quantile would be the same as ranking by
$\hat{\mu}(\vect{x}) + \kappa\,\hat{\sigma}(\vect{x})$ with
$\kappa = \Phi^{-1}(q) \approx 1.96$. Using the quantile directly is safer than
the $\hat{\mu}+\kappa\hat{\sigma}$ form, because TabPFN's predictive
distribution is a skewed, heavy-tailed histogram: the quantile always lands on
a value the model actually considers possible, whereas adding
$\kappa\,\hat{\sigma}$ can overstate the exploration bonus when the
distribution is far from Gaussian.

The predictive distribution comes from an ensemble of four estimators, each conditioned on the same table under a different ordering and preprocessing of the features. The table is refit and every candidate is rescored at each expansion step. Until $10$ arms have returned a reward, there are too few points to fit TabPFN, so the oracle samples from $P_0$ instead.

\section{Experiments}
\label{sec:experiments}

We evaluate IMABO across discrete, continuous, and mixed HPO search spaces using synthetic, machine-learning benchmarks and an online question-answering task.

\subsection{Toy Problems}
\label{sec:toy-problems}
We consider four synthetic benchmark functions, listed in \cref{tab:toy-functions} and illustrated in two dimensions in \cref{fig:toy-functions}. \textit{Sin1}, \textit{Garland} and \textit{Rastrigin} are additively separable: each is a sum of $d$ copies of the same one-dimensional profile, and the three differ only in how oscillatory that profile is. The \textit{Gaussian} function is \emph{not} separable: it places two modes on the main diagonal, a local one at $x_\ell\vect{1}$ and a better global one at $x_g\vect{1}$, where $\vect{1}$ is the all-ones vector.
\begin{table}[htbp]
    \caption{Toy functions}
    \centering
    \begin{tabular*}{\linewidth}{l @{\extracolsep{\fill}} l c}
    \toprule
    \textbf{Function} & \textbf{Formula} & \textbf{Domain} \\
    \midrule
    \textbf{Sin1} & $\displaystyle \sum_{i=1}^{d} \frac{\sin(13x_i)\sin(27x_i) + 1}{2}$ & $[0,1]^d$ \\
    \addlinespace[0.8em]
    \textbf{Garland} & $\displaystyle \sum_{i=1}^{d} 4x_i(1-x_i)\left(0.75 + 0.25\left(1 - \sqrt{|\sin(60x_i)|}\right)\right)$ & $[0,1]^d$ \\
    \addlinespace[0.8em]
    \textbf{Rastrigin} & $\displaystyle d + \frac{1}{40} \sum_{i=1}^{d} \bracket{10 \parenthese{\cos(2\pi x_i) - 1} - x_i^2}$ & $[-5.12,5.12]^d$ \\
    \addlinespace[0.8em]
    \textbf{Gaussian} & $\displaystyle 0.9\exp\!\Big(\!-\frac{\sum_{i=1}^d (x_i - x_g)^2}{2\sigma_d^2}\Big) + 0.6\exp\!\Big(\!-\frac{\sum_{i=1}^d (x_i - x_\ell)^2}{2\sigma_d^2}\Big)$ & $[0,1]^d$ \\
    \bottomrule
    \end{tabular*}
    \vspace{0.4em}
    {\footnotesize For \textit{Gaussian}: $\sigma_d = 0.35\,(x_g - x_\ell)\sqrt{d/2}$, with $x_\ell = \tfrac{1}{2\pi}$ and $x_g = 1 - \tfrac{1}{2\pi}$.\par}
    \label{tab:toy-functions}
\end{table}

\begin{figure}[htbp]
    \centering
    \includegraphics[width=\linewidth]{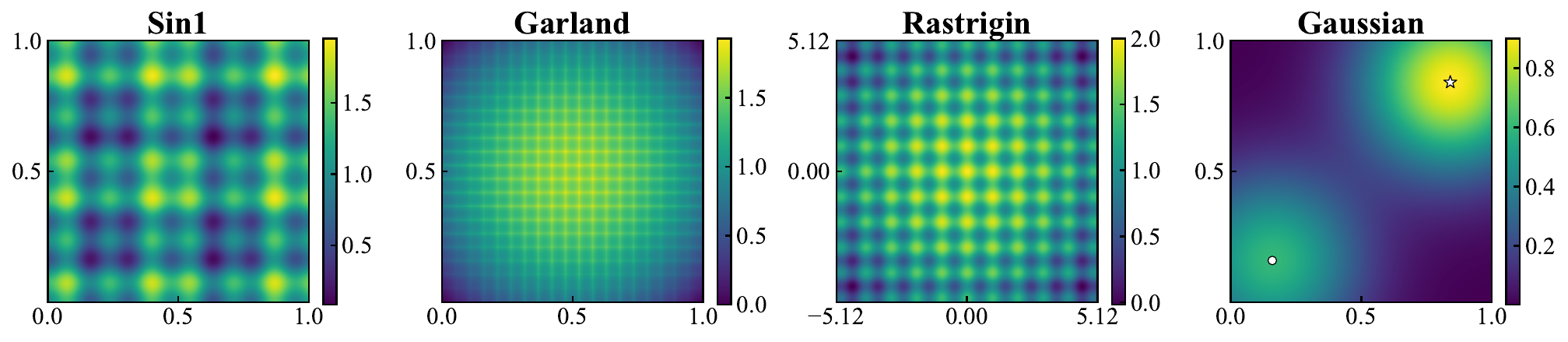}
    \caption{Illustration of the toy functions of
    \cref{tab:toy-functions} in two dimensions, over their native domains. The lattice structure of \textit{Sin1}, \textit{Garland} and \textit{Rastrigin} reflects their separability. The \textit{Gaussian} instead has two modes on the diagonal, the local one marked $\circ$ and the global one $\star$, and is separable in neither coordinate.}
    \label{fig:toy-functions}
\end{figure}

Since \textit{Sin1}, \textit{Garland} and \textit{Rastrigin} are sums of $d$ one-dimensional components, the scale of $f(\vect{x})$ grows linearly with $d$, so we work throughout with the normalized objective
\begin{equation}
    \label{eq:toy-normalization}
    g(\vect{x}) = \frac{f(\vect{x})}{d}
    \quad\text{for \textit{Sin1}, \textit{Garland}, \textit{Rastrigin},}
    \qquad
    g(\vect{x}) = f(\vect{x})
    \quad\text{for \textit{Gaussian},}
\end{equation}
since the \textit{Gaussian} function already takes values in $[0,1]$. At each evaluation, the algorithm observes a noisy value of $g$. For the three separable functions the noise is additive
\begin{equation}
    \label{eq:noise-model}
    r(\vect{x}) = g(\vect{x}) + \varepsilon,
    \qquad
    \varepsilon\sim\mathcal N(0,\sigma^2),
\end{equation}
whereas for the \textit{Gaussian} function $g(\vect{x})$ is a success probability and a pull returns a single binary observation, $r(\vect{x})\sim\operatorname{Bernoulli}\!\parenthese{g(\vect{x})}$, matching the observation model used in this paper.
The algorithms use only the noisy observations \(r(\vect{x})\). When reporting results,
we compute regret on the noiseless normalized objective.
The cumulative regret is defined as
\begin{equation*}
    \mathcal R_{T}
    =
    \sum_{t=1}^{T} \parenthese{g^* - g(\vect{x}_t)}
\end{equation*}
where $g^*=\max_{\vect{x}}g(\vect{x})$ and $\vect{x}_t$ is the configuration evaluated at step $t$. For the separable functions, $g^*$ is the optimum of a single one-dimensional component, so this normalization makes regret values comparable across dimensions.

\subsection{Comparison Against Baseline Algorithms}
\label{sec:toy-xarm-comparison}
We next compare IMOSS-TPE with three hierarchical optimistic optimization baselines, StoSOO~\cite{valko2013stochastic}, HOO-T~\cite{bubeck2011xarmed}, and StroquOOL~\cite{bartlett2019simple}, and with the factored baseline Hier-MAB~\cite{fu2024autorag}, which requires finite axes and runs on a grid of $11$ evenly spaced values per coordinate of the native domain (Hier-MAB-11). We consider the four-dimensional versions of all four functions of \cref{tab:toy-functions}. For each method, we use evaluation budgets $T \in \{1000,3000,5000,10000\}$ and repeat every experiment with $20$ independent random seeds. Observations follow the model of \cref{eq:noise-model}: additive Gaussian noise for \textit{Sin1}, \textit{Garland} and \textit{Rastrigin}, and a single Bernoulli draw for the \textit{Gaussian}. Regrets are computed using the corresponding noiseless objective.

The tree-based baselines construct their hierarchical partitions over the unit hypercube $[0,1]^d$, whereas the toy functions are defined on their native domains: $[0,1]^d$ for \textit{Sin1}, \textit{Garland} and the \textit{Gaussian}, and $[-5.12,5.12]^d$ for \textit{Rastrigin}. To ensure that all methods optimize over the same domain, each suggestion $\vect{z}\in[0,1]^d$ produced by a tree-based method is mapped back to the native search space before evaluation,
\begin{equation}
    \label{eq:toy-xarm-mapping}
    x_i = \ell_i + (u_i-\ell_i)z_i,
\end{equation}
where $[\ell_i,u_i]$ denotes the native range of coordinate $i$. The observed reward is then fed back to the tree-based method as the reward of the corresponding normalized point $\vect{z}$. This affine transformation preserves the hierarchical partition on $[0,1]^d$ while ensuring that objective values and regrets are evaluated over the same search space as IMOSS-TPE.

Once each budget is exhausted, every method returns a single recommendation $\vect{x}_{(T)}$, but the recommendation rules differ. IMOSS-TPE returns the evaluated configuration with the largest empirical mean in its active set. Hier-MAB returns its per-axis best values, the combination of the empirical-mean-maximizing value on each axis. StoSOO returns the empirical-mean maximizer at the maximum depth at which a cell has been expanded. HOO-T instead returns the visited node with the largest cached optimistic $B$-value at round $T$, using the same optimism-under-uncertainty criterion that guides its search. StroquOOL separates candidate generation from final validation: it constructs a shortlist by selecting the empirical-mean maximizer among nodes satisfying progressively larger sampling thresholds, evaluates each shortlisted candidate on a fresh independent batch, and returns the candidate with the largest validation mean. Recommendations produced in normalized tree coordinates are mapped back to the native search space before computing simple regret. Consequently, the reported simple regret reflects each algorithm's native recommendation rule, rather than a common post-processing rule applied to all methods.

\begin{figure*}[htbp]
    \centering
    \includegraphics[width=\textwidth]{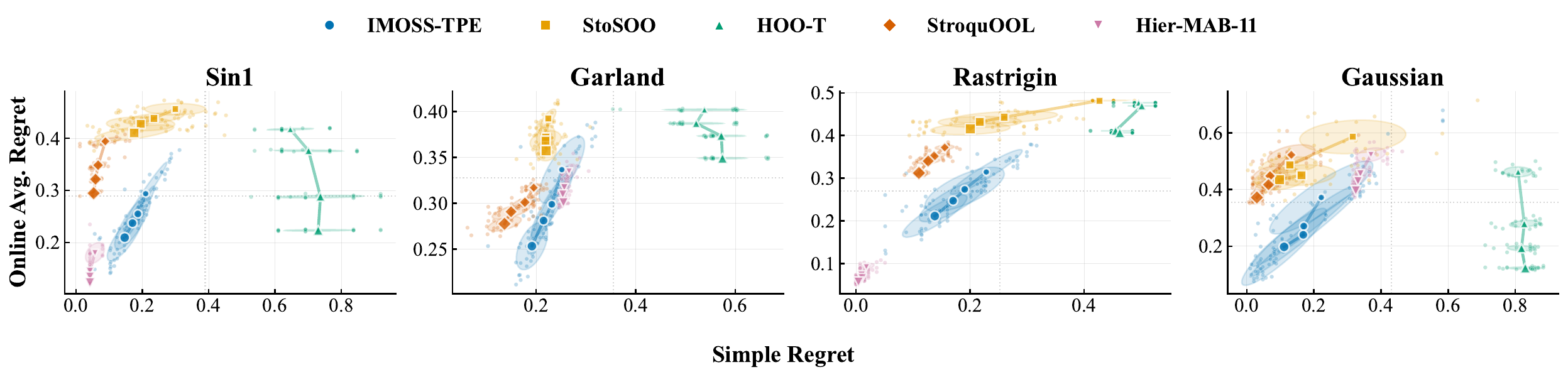}
    \caption{Trade-off between simple regret and online average regret $\overline{\mathcal{R}}_T$ (terminal value of the running average at budget $T$) on
    \textit{Sin1}, \textit{Garland}, \textit{Rastrigin} and \textit{Gaussian} in 4D. We report four evaluation budgets
    $T \in \{1000,3000,5000,10000\}$, with larger markers corresponding to
    larger budgets. Small dots show individual runs, and ellipses represent
    the empirical covariance across $20$ repetitions.}
    \label{fig:toy-xarm}
\end{figure*}

Figure~\ref{fig:toy-xarm} reports the joint behavior of simple regret and online average regret $\overline{\mathcal{R}}_T$. Each large marker represents the mean performance at one evaluation budget, with marker size increasing with $T$. Small dots show individual runs, and the ellipses show one-standard-deviation contours of the empirical covariance across the $20$ repetitions. The lower-left region is preferable, corresponding to both an accurate final recommendation and a low average regret.

Against the three tree-based baselines, IMOSS-TPE achieves the most balanced trade-off. On \textit{Sin1} it maintains substantially lower average regret than all three while reaching competitive simple regret, and on \textit{Garland} it keeps a smaller but consistent advantage on average regret. StroquOOL often produces a better final recommendation, but at a higher average regret, whereas HOO-T exhibits persistently high simple regret. The same pattern holds on \textit{Rastrigin}: at $T=10000$, IMOSS-TPE obtains better online average regret than StroquOOL at comparable simple regret, and StoSOO and HOO-T are worse on both criteria.

Hier-MAB-11 behaves exactly as its structural assumption predicts. The first three functions are additively separable across coordinates, the ideal regime for coordinate-wise credit assignment, and on \textit{Sin1} and \textit{Rastrigin} it attains both the best recommendation and the lowest average regret of all methods. On \textit{Garland}, whose per-coordinate profile oscillates rapidly, its $11$-point grid cannot resolve the axis optima and its recommendation falls behind every method except HOO-T. On the non-separable \textit{Gaussian}, Hier-MAB-11 falls behind IMOSS-TPE on both criteria at every budget, and here the grid is not the obstacle, since it does contain a point close to the global mode: its per-axis best values settle on the grid point nearest the local mode, which no single-coordinate move can leave.

These differences are consistent with the algorithms' recommendation rules. StroquOOL explicitly revalidates shortlisted candidates and therefore favors simple regret, while HOO-T selects its recommendation by its optimistic $B$-value. IMOSS-TPE recommends the empirical-mean maximizer and uses MOSS to control the cost incurred throughout optimization. Overall, IMOSS-TPE is the only method that occupies the favorable region of the empirical Pareto frontier on all four functions. Hier-MAB-11 beats it where its factorization is exactly well-specified and its grid is fine enough, but degrades sharply on \textit{Garland} and on the non-separable \textit{Gaussian}, while the tree-based baselines never combine competitive recommendations with low online regret.

\subsection{Impact of the IMOSS bandit policy}
\label{sec:moss-allocation-impact}
We investigate the effect of the IMOSS allocation by comparing IMOSS-TPE with fixed-replication TPE variants implemented using Optuna~\cite{akiba2019optuna}. For these baselines, denoted TPE-$k$, each configuration suggested by TPE is evaluated exactly $k$ times before a new configuration is requested. Optuna observes the empirical mean
\[
    \widehat{\mu}(\vect{x})
    =
    \frac{1}{k}\sum_{j=1}^{k} r_j(\vect{x})
\]
Under an evaluation budget $T$, TPE-$k$ therefore explores approximately $T/k$ distinct configurations. In contrast, IMOSS-TPE maintains an active set of $\Theta(T^\beta)$ configurations and adaptively chooses between admitting a new configuration through TPE and re-evaluating an existing one through MOSS. The two methods therefore allocate the same budget differently: TPE-$k$ controls the number of distinct configurations through $k$, whereas IMOSS-TPE trades configuration-space coverage against the accuracy with which existing configurations are evaluated.
The two methods evaluate numbers of distinct configurations of the same order when
\[
    \frac{T}{k}\simeq T^\beta,
    \qquad\text{or equivalently}\qquad
    k\simeq T^{1-\beta}.
\]
In our experiments, we use $T=5000$ evaluations on the four-dimensional versions of these functions and $\beta=0.5$, giving a crossover near $k=71$. Thus, TPE-50 evaluates more distinct configurations than IMOSS-TPE, TPE-70 evaluates approximately the same number, and the larger-$k$ baselines evaluate fewer.

\begin{figure*}[htbp]
    \centering
    \includegraphics[width=\textwidth]{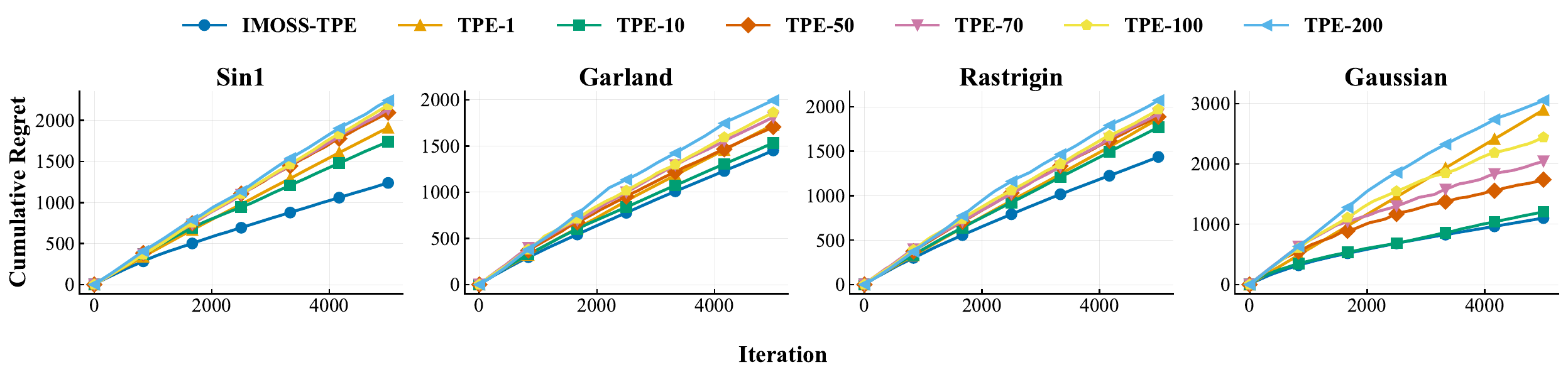}
    \caption{Cumulative regret over 5000 evaluations on the four-dimensional \textit{Sin1}, \textit{Garland}, \textit{Rastrigin} and \textit{Gaussian} functions. IMOSS-TPE is compared with fixed-replication TPE-$k$ baselines for $k\in\{1,10,50,70,100,200\}$. The observation noise follows \cref{eq:noise-model} with $\sigma=0.7$ for \textit{Sin1}, \textit{Garland} and \textit{Rastrigin}, and is a single Bernoulli draw for the \textit{Gaussian}.}
    \label{fig:cumulative-regret-comparison}
\end{figure*}

\Cref{fig:cumulative-regret-comparison} shows that IMOSS-TPE achieves lower cumulative regret than every TPE-$k$ baseline on all four functions. Importantly, this advantage persists for TPE-50 and TPE-70, whose surrogates are fitted using at least as many distinct configurations as the IMOSS-TPE surrogate. The cumulative-regret improvement therefore cannot be explained solely by the number of configurations available to TPE. Rather, it reflects the benefit of adaptive evaluation allocation: TPE-$k$ commits $k$ evaluations to every suggestion, including configurations that appear unpromising after only a few observations, whereas IMOSS reallocates subsequent evaluations according to both observed performance and uncertainty.

\subsection{HPOBench}
\label{sec:hpobench}

HPOBench~\cite{eggensperger2021hpobench} collects HPO benchmarks built from precomputed model evaluations, reporting validation performance for predefined configurations and evaluation budgets. Configurations can therefore be served repeatedly online without retraining the underlying models, and the known validation results allow exact regret computation.

\subsubsection{HPO on Discrete Parameters}
\label{sec:openml-hpo}

\begin{figure*}[!ht]
\centering
\includegraphics[width=0.9\textwidth]{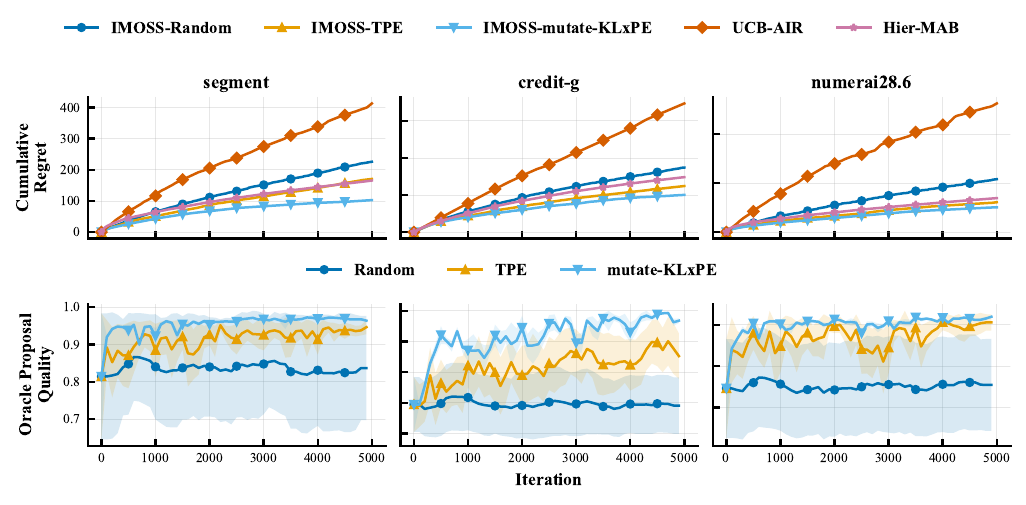}
\caption{Cumulative regret (upper row, lower is better) and true mean reward of the arms each oracle proposes (lower row, higher is better) on the three discrete OpenML HPO problems, averaged over $30$ paired runs. UCB-AIR and Hier-MAB have no oracle and appear only in the upper row.}
\label{fig:openml-hpo-cumulative-regret}
\end{figure*}

We use the random forest benchmarks provided by HPOBench for three OpenML classification tasks, \texttt{segment}, \texttt{credit-g}, and \texttt{numerai28.6}. An arm $\vect{x}$ is a random-forest configuration defined by \texttt{max\_depth}, \texttt{max\_features}, \texttt{min\_samples\_leaf}, and \texttt{min\_samples\_split}; discretizing these into evenly spaced values yields $K=2250$ arms per task. Every configuration uses $512$ trees and the full training set (subsample fraction $1$), the largest evaluation budget available in the benchmark, and an arm's mean reward $p_{\vect{x}}$ is its validation accuracy averaged over five training seeds. We turn this offline benchmark into an online one by binarization: pulling $\vect{x}_t$ at round $t$ returns $r_t \sim \operatorname{Bernoulli}(p_{\vect{x}_t})$, so the optimizer observes a noisy binary outcome rather than the true mean. Since all means are known, we set $p^\star=\max_{\vect{x}\in\mathcal{X}}p_{\vect{x}}$ and compute cumulative regret $\mathcal{R}_T$ exactly. We compare IMOSS-Random, IMOSS-TPE and IMOSS-mutate-KL$\times$PE, which share the same active-set schedule and MOSS allocation rule with $\beta=0.5$ and differ only in their oracle. UCB-AIR serves as a restart-free infinitely many-armed baseline, and Hier-MAB, the two-level hierarchical bandit of AutoRAG-HP~\cite{fu2024autorag}, as a factored one: a high-level UCB1 bandit picks which hyperparameter to perturb and one low-level UCB1 bandit per hyperparameter picks its value, while every other coordinate keeps its per-axis best value: the value with the highest mean reward on that axis alone. Crediting the reward to the perturbed coordinate therefore presumes near-separable rewards. For every task and method, we perform thirty paired runs with $T=5000$ pulls.

\Cref{fig:openml-hpo-cumulative-regret} compares cumulative regret and oracle proposal quality across the three discrete HPO tasks. Every IMABO variant accumulates less regret than UCB-AIR, including IMOSS-Random, indicating that the active-set schedule and MOSS allocation rule already help before any learned oracle is added. The two learned oracles are strongest overall and IMOSS-mutate-KL$\times$PE attains the lowest regret on all tasks: \texttt{segment}, \texttt{credit-g}, and \texttt{numerai28.6}. The lower row, measured on a copy of the optimizer state so these diagnostic proposals never alter the trajectory, explains much of this. Random proposals stay at constant quality by construction, whereas both learned oracles quickly begin proposing arms of substantially higher true mean reward.
On \texttt{credit-g}, mutate-KL$\times$PE sustains a high, stable proposal quality whereas TPE is lower and markedly more variable, consistent with a sparse optimum depending on two hyperparameters at once. On \texttt{segment} and \texttt{numerai28.6}, whose rewards vary mainly with \texttt{max\_depth}, both learned oracles reach high proposal quality without difficulty; on \texttt{numerai28.6} many arms additionally share near-identical means, so the residual gap between the oracles translates into only mild differences in accumulated regret.

\subsubsection{HPO on Continuous Parameters}
\label{sec:continuous-hpo}

\begin{figure*}[!ht]
    \centering
    \includegraphics[width=0.95\textwidth]{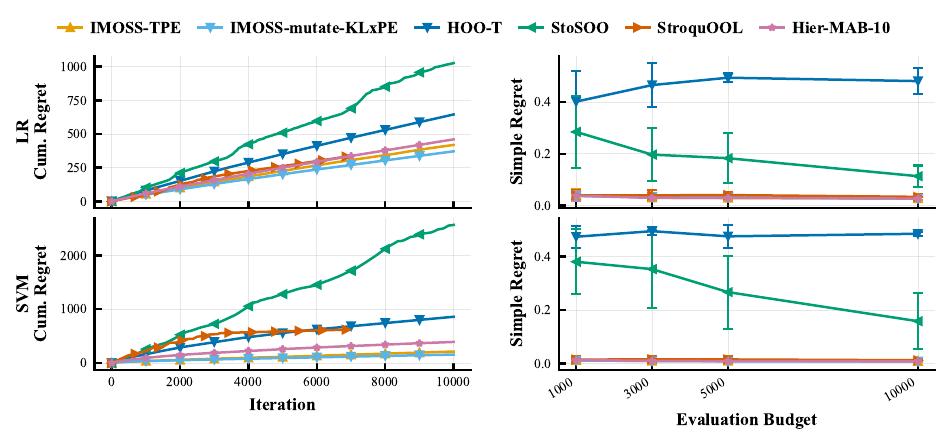}
    \caption{Cumulative regret (left) and simple regret of the recommended configuration (right) on the two continuous OpenML HPO problems, Logistic Regression (LR) and Support Vector Machine (SVM). Lower is better. StroquOOL's curve ends after $6866$ evaluations because its schedule is built to spend at most $T=10000$ evaluations, and it stops once that schedule is complete.}
    \label{fig:cumulative-regret-comparison-lr-svm-hpo}
\end{figure*}

We evaluate IMOSS-TPE and IMOSS-mutate-KL$\times$PE on the Logistic Regression and Support Vector Machine benchmarks associated with the \texttt{kr-vs-kp} dataset (OpenML task~167149).\footnote{\url{https://www.openml.org/t/167149}} Both define two-dimensional continuous search spaces with logarithmic scaling: the SGD regularization parameter $\alpha$ and the initial learning rate $\eta_0$ over $[10^{-5},1]$ for Logistic Regression, and $C$ and $\gamma$ over $[2^{-10},2^{10}]$ for Support Vector Machine, both with default benchmark settings.
For each evaluation budget $T$, we perform $20$ independent runs. We compare both IMABO variants, with $\beta=0.5$, against StoSOO~\cite{valko2013stochastic}, StroquOOL~\cite{bartlett2019simple}, and HOO-T~\cite{bubeck2011xarmed}. None of these three baselines is anytime, so we give each of them the horizon $T$ in advance, information that the IMABO family does not need. We use the horizon-aware variant HOO-T because the per-round cost of the anytime original grows quadratically with $T$, which is impractical at our budgets. We also run Hier-MAB, which requires every hyperparameter to range over an explicit finite set, with both axes discretized on geometrically spaced grids of $10$ values per axis (Hier-MAB-10). For the tree-based methods, each logarithmic search space is normalized to the unit square $[0,1]^2$ and suggested points are mapped back to the original ranges before evaluation.

\Cref{fig:cumulative-regret-comparison-lr-svm-hpo} compares cumulative regret during optimization (left) and simple regret of the final recommendation (right). The two IMABO variants, IMOSS-TPE and IMOSS-mutate-KL$\times$PE, obtain the lowest cumulative regret on both tasks without a pre-declared discretization of the axes, and both reach near-zero simple regret. On simple regret, they are matched by StroquOOL and Hier-MAB-10, which also converge to near-zero final regret, so the methods separate mainly through cumulative regret rather than through the quality of the final recommendation. We compare against HOO-T, the only baseline explicitly designed to control cumulative regret, and StroquOOL, whose simple-regret guarantees are state of the art under unknown smoothness and noise. On Logistic Regression, the IMABO variants and StroquOOL are essentially tied for the lowest cumulative regret, whereas on Support Vector Machine the IMABO variants are clearly ahead, with StroquOOL accumulating far more regret while sweeping its schedule. StoSOO incurs the highest cumulative regret on both tasks, and its simple regret, although it does decrease with the budget, remains far above that of the IMABO variants; HOO-T is the weakest on simple regret, which stays high on both tasks and does not improve with the budget. On Logistic Regression, StroquOOL ends below HOO-T on cumulative regret even though only HOO-T targets that metric. Hier-MAB-10 is competitive on Logistic Regression but stays clearly behind IMABO on Support Vector Machine.
\subsection{LLM Question Answering on HotpotQA}
\label{sec:hotpotqa}
HotpotQA~\cite{yang2018hotpotqa} is an open-domain question-answering benchmark built from Wikipedia, in which each example pairs a question and its answer with supporting facts drawn from several articles, so answering requires combining information across documents. We use the retrieval version distributed through BEIR~\cite{thakur2021beir}.
Each question is answered by a retrieval-augmented generation pipeline: a retriever selects the \texttt{top\_k} Wikipedia passages by cosine similarity, and a large language model generates the answer from a prompt containing the question and those passages. At round $t$, the optimizer selects the complete pipeline configuration $\vect{x}_t$ before observing question $q_t$, so it must find a configuration that performs well across the stream rather than adapt to each question.

Retrieval uses \texttt{all-MiniLM-L6-v2} embeddings~\cite{reimers2019sentence}, precomputed for all Wikipedia documents and normalized to unit $\ell_2$ norm, as are the encoded questions.
A configuration includes the generation model, the prompt template, the retrieval depth \texttt{top\_k} from $1$ to $10$, and the generation temperature in $[0,1]$. Prompt templates are \texttt{few\_shot} with fixed examples, \texttt{zero\_shot} with instructions but no examples, and \texttt{naive} with only the question and retrieved documents. The full model list, prompt templates, and generation details are given in \Cref{sec:hotpotqa-details}.

For each question $q$, answer quality $a_q(\vect{x})$ is the normalized token-overlap F1 between the generated and gold answers, and retrieval quality $r_q(\vect{x})$ is the title-level F1 between the retrieved and gold supporting documents. The reward combines both as \mbox{$w \cdot a_q(\vect{x})+(1-w)\cdot r_q(\vect{x})$} with $w=0.8$, and we report online average regret $\overline{\mathcal{R}}_T=\mathcal{R}_T/T$. After the stream, each method returns its best configuration, which we evaluate on a disjoint set of unseen questions and report as average per-question simple regret. For each of five seeds, we randomly permute the test questions and use $5000$ for the online stream and a disjoint $500$ for final evaluation, giving all methods the same split and question order. 
We compare IMOSS-TPE, IMOSS-mutate-KL$\times$PE and IMOSS-TabPFN with Random and Hier-MAB. Random selects configurations uniformly without using previous observations, and Hier-MAB requires explicit finite axes, so its temperature, the only continuous axis, is discretized to eleven evenly spaced values.

\begin{figure*}[!ht]
    \centering
    \includegraphics[width=\textwidth]{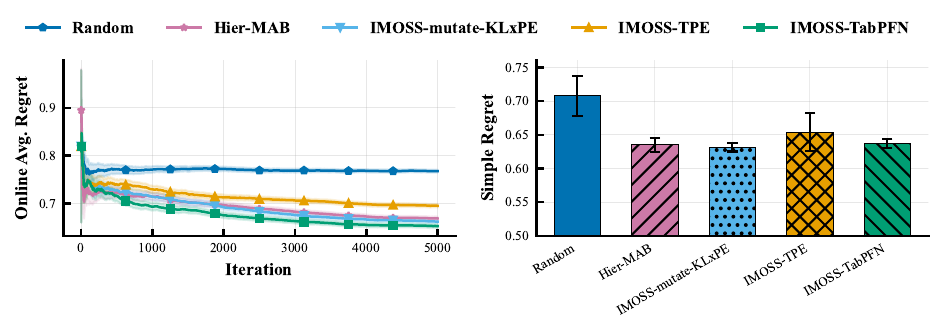}
    \caption{Online average regret and simple regret on HotpotQA, averaged over $5$ seeds; shaded bands and error bars show standard deviations across seeds. Lower is better.}
    \label{fig:hotpotqa-results-5000samples}
\end{figure*}

In \Cref{fig:hotpotqa-results-5000samples}, both $\hat{\vect{x}}$-mutation variants, IMOSS-mutate-KL$\times$PE and IMOSS-TabPFN, improve on Random on both metrics and on Hier-MAB online, and their online advantage widens as more questions are observed, consistent with the learned oracles making greater use of the accumulated observations. IMOSS-TPE, by contrast, trails Hier-MAB on both metrics. HotpotQA is a near-separable space, the regime in which coordinate-local search is strongest: the mutation oracles refine the best configuration one coordinate at a time and Hier-MAB holds every axis except one at its per-axis best, so both keep the active set concentrated around the strongest configuration. IMOSS-TPE instead admits arms from across the whole good region; although each proposal is individually reasonable, the active set is more dispersed, so more pulls land away from the best accumulated configuration, raising online regret while the best neighborhood is refined less and leaving the final recommendation behind as well.
IMOSS-TabPFN attains the lowest online average regret, with IMOSS-mutate-KL$\times$PE and Hier-MAB close behind. On simple regret the two $\hat{\vect{x}}$-mutation variants and Hier-MAB are statistically indistinguishable and all clearly below Random, so Hier-MAB matches them on the quality of the final recommendation while trailing them online. Overall, the methods that keep the active set concentrated around the strongest configuration, the two mutation oracles and Hier-MAB, dominate on both metrics, while IMOSS-TPE, the only method that proposes across the whole good region, trails even Hier-MAB; online average regret separates more sharply than simple regret does.

\section{Practical Usage}
\label{sec:practical}

We propose extensions of IMABO that address specific problems encountered when running it in production.
\subsection{Delayed and censored feedback}
In many applications the reward is not available the moment a configuration is tried. For instance, it may depend on user feedback that follows the interaction, and this feedback may arrive after an unknown delay or fail to arrive at all. The policy must keep selecting configurations without waiting for these rewards, so at any round it faces a mixture of pulls that have already returned a reward and pulls that are still pending. Write $n^{\text{rew}}_{\vect{x}}$ and $n^{\text{pend}}_{\vect{x}}$ for the completed and pending pulls of arm $\vect{x}$, and $N^{\text{rew}}$ and $N^{\text{pend}}$ for their totals over all arms. Let $p \in (0,1]$ be the fraction of pulls that eventually return a reward, estimated from the logs. IMABO counts each pending pull as a fraction $p$ of a reward, which lets both the expansion rule and the index run on the rewards that the pending pulls will produce in expectation,
\begin{equation}
    \label{eq:delayed-feedback}
    \tilde{N}_t = N^{\text{rew}} + p\, N^{\text{pend}},
    \qquad
    \widetilde{\mathcal{B}}_t(\vect{x}) = \hat{\mu}_{\vect{x}} + \sqrt{\frac{1+\alpha}{2}\cdot\frac{\max\!\big(0,\ \log\!\big(\tilde{N}_t / (\tilde{N}_t^{\,\beta}\, \tilde{n}_{\vect{x}})\big)\big)}{\tilde{n}_{\vect{x}}}},
\end{equation}
where $\tilde{n}_{\vect{x}} = n^{\text{rew}}_{\vect{x}} + p\, n^{\text{pend}}_{\vect{x}} + 1$ and $\hat{\mu}_{\vect{x}}$ averages the completed rewards of the arm. A new arm enters whenever $|\mathcal{M}_t| < \tilde{N}_t^{\,\beta}$. With $p = 1$ and $N^{\text{pend}} = 0$ every pull has already returned its reward and the rule reduces to~\eqref{eq:moss}. A pull that remains pending beyond a fixed window is dropped from the count.

\subsection{Experimental evaluation}
\label{sec:delayed-feedback-experiment}

We evaluate the delay-aware switching rule above against a delay-oblivious
variant and a no-delay skyline on LCBench, accessed through the YAHPO-Gym
surrogate, which models AutoPyTorch multilayer-perceptron training on OpenML
tabular tasks~\cite{zimmer2021autopytorch}. 
We use three instances:
\texttt{higgs} (167200), \texttt{APSFailure} (168868), and
\texttt{Fashion-MNIST} (189908). Each defines the seven-dimensional search
space of \Cref{tab:lcbench-search-space}. The surrogate returns a validation
accuracy in $[0,100]$, rescaled to a success probability $p \in [0,1]$; the reward for each pull is a Bernoulli$(p)$ random variable. As the maximum achievable accuracy $p^\star$ has no
closed form on a continuous space, we estimate it once per instance from a dense random sample of $100{,}000$ configurations.

\begin{table}[htbp]
    \caption{LCBench search space (seven hyperparameters, mixed
    continuous and finite). Bounds and log-scale flags are read from the
    surrogate's configuration space.}
    \centering
    \small
    \begin{tabular*}{\linewidth}{l @{\extracolsep{\fill}} ll}
    \toprule
    \textbf{Hyperparameter} & \textbf{Type} & \textbf{Range} \\
    \midrule
    \texttt{num\_layers}    & Categorical & $\{1,\dots,5\}$ \\
    \texttt{batch\_size}    & Integer (log-scaled) & $[16, 512]$ \\
    \texttt{max\_units}     & Integer (log-scaled) & $[64, 1024]$ \\
    \texttt{learning\_rate} & Continuous (log-scaled) & $[10^{-4}, 10^{-1}]$ \\
    \texttt{momentum}       & Continuous & $[0.1, 0.99]$ \\
    \texttt{weight\_decay}  & Continuous & $[10^{-5}, 10^{-1}]$ \\
    \texttt{max\_dropout}   & Continuous & $[0.0, 1.0]$ \\
    \bottomrule
    \end{tabular*}
    \label{tab:lcbench-search-space}
\end{table}

Each configuration's feedback delay is set by its own predicted training time,
returned in seconds alongside the accuracy by the surrogate. To express this
runtime in simulator steps, we divide it by a fixed per-instance constant
$\tau$,
\begin{equation*}
    \texttt{steps}(\vect{x}) = \frac{\texttt{runtime\_seconds}(\vect{x})}{\tau},
    \qquad
    \tau = \frac{\tilde{t}}{6},
\end{equation*}
where $\tilde{t}$ is the median predicted runtime over the
$100{,}000$-configuration sample. Choosing $\tau = \tilde{t}/6$ fixes the
step scale by anchoring it to the median configuration, whose runtime is
$\tilde{t}$ by definition. A stochastic delay is then obtained by scaling
$\texttt{steps}(\vect{x})$ with a multiplicative log-normal jitter that
represents queueing noise (a job with predicted runtime $t$ does not always take
exactly $t$),
\begin{equation*}
    \texttt{delay}(\vect{x}) =
    \texttt{delay\_scale} \times \texttt{steps}(\vect{x})
    \times \xi,
    \qquad
    \xi \sim \mathrm{LogNormal}(0, \sigma^2),
\end{equation*}
with $\sigma = 0.5$ (median jitter $1$) and $\texttt{delay\_scale} = 1$.
Feedback is then lost in two independent ways: a configuration-independent
Bernoulli filter admits a pull only with probability
$\texttt{feedback\_freq} = 0.2$ (crashed or never-submitted jobs, $80\%$
dropped outright), and a patience window drops any admitted pull whose delay
exceeds the $q = 0.95$ quantile of that instance's delay distribution
($32$, $33$, and $19$ steps for \texttt{higgs}, \texttt{APSFailure}, and
\texttt{Fashion-MNIST}). A pull is therefore never observed with probability
\begin{equation*}
    \mathbb{P}(\text{never observed})
    = (1 - \texttt{feedback\_freq})
    + \texttt{feedback\_freq} \cdot \mathbb{P}(\texttt{delay} >
    \texttt{patience}).
\end{equation*}
%
We consider three variants of IMOSS-TPE: \emph{IMOSS-TPE
Delayed} is the delay-aware switching rule above, run under delayed and censored feedback;
\emph{IMOSS-TPE Naive} is the identical optimizer exposed to the same delayed and censored
environment but ignoring pending pulls, isolating the cost of
delay-blindness; \emph{IMOSS-TPE} is that same optimizer under instant
feedback, the no-delay skyline. \emph{UCB-AIR} is an infinitely many-armed bandit
reference, also run under instant feedback. All four methods use $\beta = 0.5$.

\begin{figure*}[!ht]
    \centering
    \includegraphics[width=\textwidth]{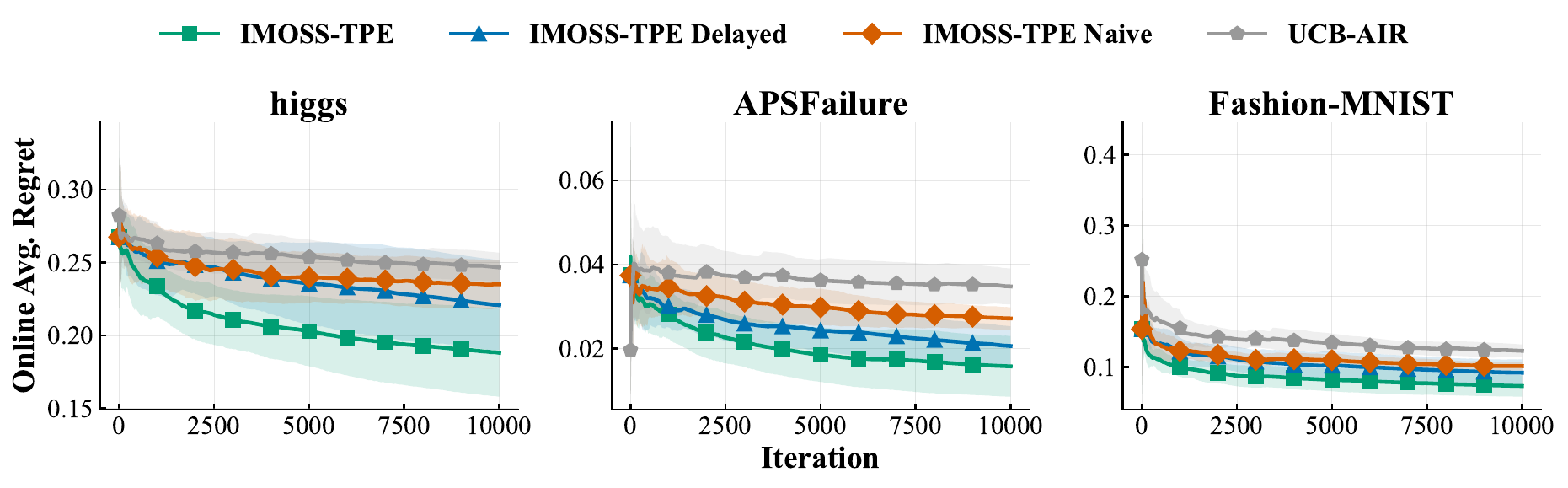}
    \caption{Delayed and censored feedback on three LCBench instances
    (\texttt{higgs}, \texttt{APSFailure}, \texttt{Fashion-MNIST}). Each panel
    plots online average regret $\overline{\mathcal{R}}_T$; lower is better. Curves
    are means over $10$ seeds and shaded bands are $\pm 1$ standard deviation.
    Feedback is delayed by each configuration's predicted training time (median
    $6$ steps, log-normal jitter $\sigma = 0.5$) and censored by a Bernoulli
    filter ($\texttt{feedback\_freq} = 0.2$) plus a per-instance $q = 0.95$
    patience window ($32$, $33$, $19$ steps), so only about $19\%$ of pulls ever
    return a reward. \emph{IMOSS-TPE} (green) is the no-delay skyline and
    \emph{UCB-AIR} (grey) is a no-delay reference baseline; \emph{IMOSS-TPE
    Delayed} (blue) and \emph{IMOSS-TPE Naive} (orange) run the same optimizer
    under delayed feedback, with and without the delay-aware switching rule.}
    \label{fig:cumulative-delayed-feedback}
\end{figure*}

\Cref{fig:cumulative-delayed-feedback} reports online average regret over
$T = 10000$ iterations with $10$ seeds. The ordering is consistent across
all three instances: the delay-aware rule sits between the delay-blind Naive
variant and the no-delay skyline, closing most of the gap to the skyline. The
two delayed methods share the environment and receive the same censored stream
--- only about $19\%$ of pulls ever return a reward --- so the gap between them
is caused by the switching rule alone.

The mechanism is made explicit in \Cref{fig:active-set}. The Naive rule admits
a new arm whenever $|\mathcal{M}_t| < t^{\beta}$, using the raw step count, so
at the horizon it holds the full $t^{\beta} = \sqrt{T} \approx 100$ arms even
though most of those steps never returned a reward. The delay-aware rule
instead throttles admission by the effective count $\tilde{N}_t$
of~\eqref{eq:delayed-feedback}: because roughly $80\%$ of pulls never arrive,
$\tilde{N}_t \ll t$ and the active set settles at about $60$ arms. Both delayed
methods observe the same $\approx\!1{,}900$ rewards, so spreading them over
fewer arms substantially raises the observations per arm for the delay-aware
rule over the Naive one, though both stay well below the skyline, which
observes every pull. More observations per arm means tighter
arm-mean estimates, which is what lets the delay-aware rule rank candidates
reliably and recover most of the distance to the skyline, while the Naive rule
pays for exploring arms it cannot resolve. Despite discarding roughly $80\%$ of
feedback, the delay-aware optimizer still finishes below UCB-AIR, which sees
instant feedback.

\begin{figure*}[!ht]
    \centering
    \includegraphics[width=\textwidth]{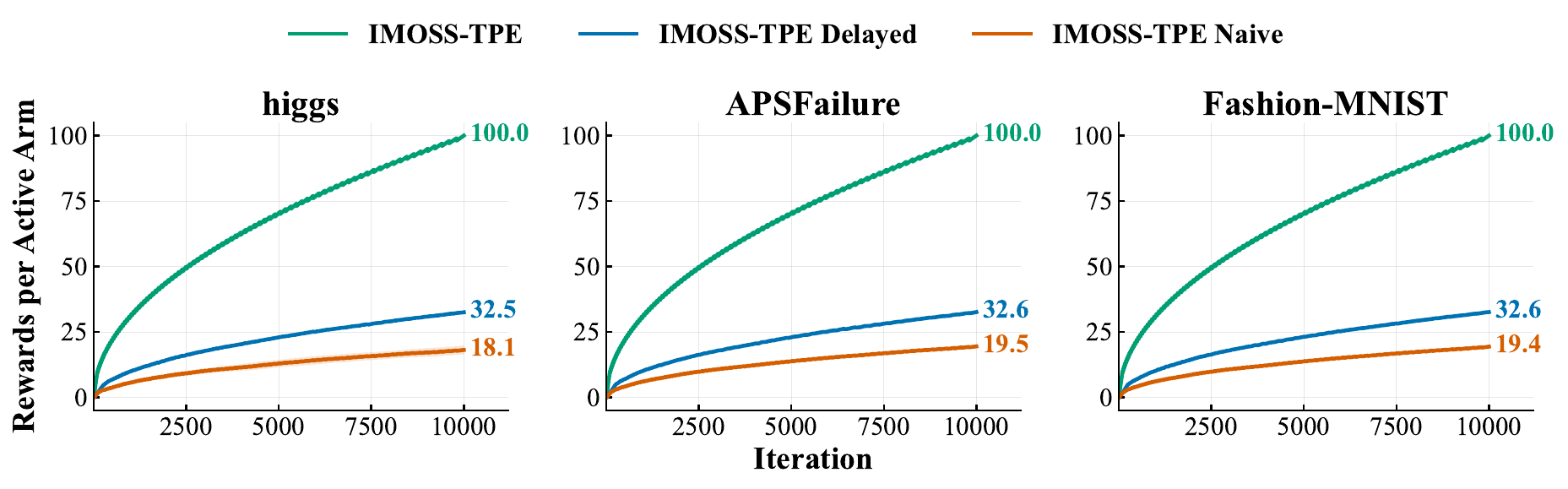}
    \caption{Observed rewards per active arm on the three LCBench instances:
    the cumulative number of rewards delivered so far divided by the number of
    active arms $|\mathcal{M}_t|$, against iteration. Under identical
    delayed and censored feedback, the two delayed methods observe the same total
    number of rewards ($\approx\!1{,}900$), so this ratio measures how
    concentrated that feedback is across the active set. The delay-aware rule
    (\emph{IMOSS-TPE Delayed}, blue) throttles arm admission by the effective
    count $\tilde{N}_t$ and holds $\approx\!60$ active arms, giving $\approx\!33$
    rewards per arm; the delay-blind \emph{IMOSS-TPE Naive} (orange) expands to
    the full $\approx\!100$ arms and reaches only $\approx\!18$. The no-delay
    skyline (\emph{IMOSS-TPE}, green) observes every pull, giving $100$ rewards
    per arm. Curves are means over $10$ seeds.}
    \label{fig:active-set}
\end{figure*}

\paragraph{Restart on an evolving search space.} A search that runs for a long time will usually see its space of configurations grow, whether because a provider ships a new model or because an existing parameter is allowed a new value. We would like IMABO to give those late additions as much attention as it gave the first configurations it ever saw. It does not do so by default, because it explores the most at the beginning of a run. The condition $|\mathcal{M}_t| < t^{\beta}$ is easy to meet when $t$ is small and hard to meet once $t$ is large, so newly added configurations, arriving when $t$ is already large, would be picked up only rarely. We correct this by restarting the schedule at the round $t_0$ at which the space changed. With $N_0 = |\mathcal{M}_{t_0}|$ arms already in place, a new arm enters when
\begin{equation*}
    |\mathcal{M}_t| - N_0 < (t - t_0)^{\beta}.
\end{equation*}
Only the schedule is reset, and all past rewards are preserved, so when the oracle proposes configurations from the larger space it still uses everything learned so far, unlike a true restart, which would discard that knowledge.

\section{Conclusion}


We addressed OHPO by framing it as an infinitely many-armed bandit over a mixed and possibly conditional search space. We introduced IMABO, a framework that combines a bandit policy for allocating pulls with an oracle for suggesting new configurations, and IMOSS, an anytime and restart-free policy with a cumulative quantile-regret bound. When instantiated with a Random, a TPE, a mutate-KL$\times$PE, or a TabPFN oracle, IMABO outperformed its baselines on both classical machine-learning models and a language-model agent while being computationally efficient for online optimization.

Several features set IMABO apart among online tuners. It operates directly on mixed, conditional, tree-structured spaces without discretizing any coordinate, and it comes with a cumulative quantile-regret guarantee. Because its active set grows only as $t^{\beta}$, it requests fewer distinct configurations than rounds, an advantage when adopting a configuration carries a training or warm-up cost. Its heuristic extension to delayed and censored feedback also performs well under production conditions.

A central challenge is the imbalance induced by adaptive sampling: promising configurations receive many observations, whereas most remain weakly estimated. This process can also introduce survivorship bias, since configurations favored by unusually strong early rewards are more likely to remain active and accumulate further evidence, making their true performance difficult to estimate. Developing uncertainty-aware or debiased estimators tailored to this adaptive data collection process is therefore a promising direction.

The modularity of IMABO opens several further directions. When request context is available before selection, the allocation policy could be replaced by a contextual bandit without changing the proposal mechanism. A large language model could also serve as the oracle, reading the best configurations found so far together with a description of the parameters and returning the next one to admit, and the same mechanism would extend the search to prompts themselves \cite{yang2024largelanguagemodelsoptimizers,zhou2023largelanguagemodelshumanlevel}. Constrained or multi-objective settings, such as maximizing quality subject to a cost or latency budget, are a natural extension. More broadly, jointly learning when to expand, what to propose, and which active configuration to serve could further improve the balance between discovery and exploitation.
\newpage
\bibliographystyle{unsrtnat}  
\bibliography{references}  
\clearpage
\appendix
\crefalias{section}{appendix}
\centerline{\large\bfseries\scshape Appendix}
\vspace{1em}
\section{Theoretical Analysis}
\label{sec:theory-analysis}
In this section, we prove the main theoretical result in two stages. First, we
establish \Cref{lem:growing-moss-allocation}, which extends the usual MOSS allocation bound to a growing active set. The regular growth condition
\begin{equation*}
\frac{t}{K_t} \le \lambda \frac{T}{K_T}
\end{equation*}
allows us to sandwich the time-varying MOSS confidence radius between two fixed-$K_T$ radii. We can then use the standard MOSS underestimation/overestimation decomposition and peeling argument to obtain a gap-dependent bound, which is converted into the distribution-free guarantee $O(\sqrt{K_TT})$. Importantly, the lemma holds relative to a reference arm $\vect{x}_g$ that is admitted at a possibly random time $\tau$ and remains active thereafter.

We then apply this lemma to prove \Cref{thm:quantile-regret}. Let
$\tau_\rho$ be the first time a top-$\rho$ configuration enters the active set,
and take the first such configuration as the reference arm $\vect{x}_g$. We
split the $\rho$-regret as
\begin{equation*}
\mathcal{R}_T^{\rho,+}
\le
\mathcal{R}_T^{\mathrm{disc}}
+\mathcal{R}_T^{\mathrm{expand}}
+\mathcal{R}_T^{\mathrm{MOSS}},
\end{equation*}
corresponding to three terms: the cost accumulated before $\tau_\rho$, the cost
of expansion rounds after $\tau_\rho$, and the cost of MOSS exploitation rounds
after $\tau_\rho$. We bound each in turn.
The \emph{discovery} term is controlled by the conditional coverage assumption,
which bounds the tail probability of $\tau_\rho$ and gives $O(p_\rho^{-1/\beta})$.
The \emph{expansion} term counts rounds that admit a new arm: each such round
pulls a single arm, incurring cost at most $1$, and there are at most $K_T$ such rounds, so
this term is $O(K_T)=O(T^\beta)$.
The \emph{exploitation} term is where the lemma enters: $\vect{x}_g$ remains
active and satisfies $\mu_{\vect{x}_g}\geq\mu_\rho$, so the regret relative to
$\mu_\rho$ is bounded by the regret relative to $\vect{x}_g$, and
\Cref{lem:growing-moss-allocation} gives $O(\sqrt{K_TT})=O(T^{(1+\beta)/2})$.
Finally, since $\beta<1$, the expansion term $T^\beta$ is dominated by the
exploitation term $T^{(1+\beta)/2}$, yielding the claimed rate.

Let each configuration $\vect{x} \in \xcal$ have an unknown reward
distribution supported on $[0,1]$ with mean $\mu_{\vect{x}}$. For $\rho \in (0,1)$, let $\mu_\rho$ denote the top-$\rho$
threshold and define
\begin{align*}    
\mathrm{TOP}_\rho := \curlybracket{\vect{x} \in \xcal : \mu_{\vect{x}} \ge \mu_\rho}.
\end{align*}
We consider the $\rho$-regret defined as:
\begin{align*}
\mathcal{R}_T^{\rho,+} := \sum_{t=1}^T \parenthese{\mu_\rho-\mu_{\vect{x}_t}}_+,  \qquad \parenthese{a}_+ := \max\curlybracket{a,0}.
\end{align*}

Let $\mathcal{M}_t$ be the active set after round $t$, and let
$K_t=|\mathcal{M}_t|$. The expansion schedule satisfies $K_t=\Theta(t^\beta)$
with $\beta\in(0,1)$; we use $K_t=O(t^\beta)$ and $K_t=\Omega(t^\beta)$ wherever
the corresponding constant is needed.
At each round, IMOSS either admits a new arm (an \emph{expansion} round) or
serves an arm already in $\mathcal{M}_t$ through the MOSS index (an
\emph{exploitation} round). We write $\mathcal{T}_{\mathrm{expand}}$ for the set
of expansion rounds and $\mathcal{T}_{\mathrm{MOSS}}$ for the set of
exploitation rounds, so that every round $t$ belongs to exactly one of the two.

\begin{lemma}[Allocation regret of growing-active-set MOSS]
\label{lem:growing-moss-allocation}
Let $(\mathcal{F}_t)_{t\ge0}$ be the filtration generated by the rounds up to
and including round $t$, and let $(\mathcal{M}_t)_{t\le T}$ be a possibly random,
nondecreasing sequence of active sets, with $K_t = |\mathcal{M}_t|$ and $K=K_T$.
Assume that the arm admitted or pulled at round $t$ is
$\mathcal{F}_{t-1}$-measurable, and that, conditionally on $\mathcal{F}_{t-1}$,
each new observation from arm $\vect{x}$ is independent, supported on $[0,1]$,
and has mean $\mu_{\vect{x}}$.

Assume that the active-set schedule satisfies, for some constant $\lambda \ge 1$,
\begin{align*}
    \frac{t}{K_t} \le \lambda \frac{T}{K}
    \qquad \forall t\le T
\end{align*}
Assume that, at exploitation rounds, an arm can be selected by the
allocation policy only after it has received at least one observation. At an
exploitation round $t$, the policy selects
\begin{align*}
    \vect{x}_t \in \arg\max_{\vect{x}\in\mathcal{M}_t}
    \curlybracket{
        \widehat \mu_{\vect{x}}(n_{\vect{x}}(t))
        +
        \sqrt{
        \frac{1+\alpha}{2}
        \frac{\max\curlybracket{0,\log\parenthese{t/\parenthese{K_t n_{\vect{x}}(t)}}}}
                {n_{\vect{x}}(t)}
        }
    }
\end{align*}
Let $\tau\leq T$ be an $(\mathcal{F}_t)$-stopping time and let $\vect{x}_g$ be an
$\mathcal{F}_{\tau-1}$-measurable arm, i.e., chosen using only information
available up to round $\tau-1$. Assume that $\vect{x}_g$ is active from
time $\tau$ onward, and define
\begin{align*}
    \Delta_{\vect{x}}^{\vect{x}_g} := (\mu_{\vect{x}_g}-\mu_{\vect{x}})_+.
\end{align*}
Then there exists a constant $C_{\alpha,\lambda}>0$, depending only on $\alpha$
and $\lambda$, such that
\begin{align}
    \E\bracket{\sum_{t=\tau}^T \Delta_{\vect{x}_t}^{\vect{x}_g} \mathbf{1}\curlybracket{t \in \mathcal{T}_{\mathrm{MOSS}}}}
    \le
    C_{\alpha,\lambda}\sqrt{KT}
\end{align}
\end{lemma}

\begin{proof}
The argument below is conditional on $\mathcal{F}_{\tau-1}$, which fixes the
comparator $\mu_{\vect{x}_g}$ and hence the gaps $\Delta_{\vect{x}}$ and their ordering.
Conditionally on $\mathcal{F}_{\tau-1}$ the per-arm reward streams from time
$\tau$ onward remain independent with means $\mu_{\vect{x}}$ by assumption, so the
concentration bounds below apply. Since the resulting bound is uniform in
$\mathcal{F}_{\tau-1}$, taking expectations at the end gives the stated
unconditional bound.

Let $K:=K_T$. Since the reference arm $\vect{x}_g$ is fixed throughout the proof,
we drop the superscript and abbreviate the gap of arm $\vect{x}$ relative to
$\vect{x}_g$ as
\begin{align*}
    \Delta_{\vect{x}} := \Delta_{\vect{x}}^{\vect{x}_g} = (\mu_{\vect{x}_g}-\mu_{\vect{x}})_+.
\end{align*}
For each arm $\vect{x}$ satisfying $\mu_{\vect{x}}\ge \mu_{\vect{x}_g}$, we have
$\Delta_{\vect{x}}=0$ and it does not contribute to the regret. Order the arms in
$\mathcal{M}_T$ by increasing gap,
\begin{align*}
    0=\Delta_1\le \Delta_2\le\cdots\le \Delta_K,
\end{align*}
where the reference arm $\vect{x}_g$ is one of the zero-gap arms. We identify
each arm with its rank $k$ in this ordering, so that $\Delta_k$, $\mu_k$,
$z_k$, and $N_k(T)$ all refer to the $k$-th arm. Throughout the proof, we write
\begin{align*}
    \overline{\log}(u):=\max\curlybracket{1,\log u}
\end{align*}
and define the regret under study as
\begin{align*}
    \mathcal{R}_{\vect{x}_g}(T)
    :=\sum_{t=\tau}^{T} \Delta_{\vect{x}_t}
    \mathbf{1}\curlybracket{t\in\mathcal{T}_{\mathrm{MOSS}}}.
\end{align*}
    
\paragraph{Step 1: confidence-radius sandwich.}

For $s\ge 1$, define the MOSS bonus
\begin{align*}
b_t(s) := \sqrt{\frac{1+\alpha}{2} \frac{\max\curlybracket{0,\log\parenthese{t/\parenthese{K_t s}}} }{s}}
\end{align*}
so that the MOSS index of~\eqref{eq:moss} evaluated at arm $\vect{x}$ (the index $\mathcal{B}_t(\vect{x})$ of Algorithm~\ref{alg:moss}) reads
\begin{align*}
B_t(\vect{x}) := \widehat\mu_{\vect{x}}(n_{\vect{x}}(t)) + b_t(n_{\vect{x}}(t)),
\end{align*}
where $n_{\vect{x}}(t)$ is the number of times $\vect{x}$ has been pulled before round $t$. We write
$\widehat\mu_{\vect{x},s}$ for the empirical mean over the first $s$ pulls of $\vect{x}$, so
that $\widehat\mu_{\vect{x}}(n_{\vect{x}}(t))=\widehat\mu_{\vect{x},n_{\vect{x}}(t)}$.
Define two fixed-$K$ envelope bonuses:
\begin{align*}
b_t^{-}(s) := \sqrt{\frac{1+\alpha}{2} \frac{\max\curlybracket{0,\log\parenthese{t/\parenthese{K s}}} }{s}}
\end{align*}
and
\begin{align*}
b_T^{+}(s) := \sqrt{\frac{1+\alpha}{2} \frac{\max\curlybracket{0,\log\parenthese{\lambda T/K s}} }{s}}
\end{align*}
Since $K_t\le K_T=K$, we have $t/(K_t s)\ge t/(Ks)$, which implies
$b_t(s)\ge b_t^{-}(s)$. Moreover, the regular growth condition gives
$t/(K_t s)\le \lambda T/(Ks)$, and therefore $b_t(s)\le b_T^{+}(s)$. Thus, for
all $t\le T$ and all $s\ge1$,
\begin{align*}
    b_t^-(s)\le b_t(s)\le b_T^+(s).
\end{align*}
        
\paragraph{Step 2: reduction to arms with non-negligible gaps.}
Fix $k_0\in\curlybracket{1,\ldots,K-1}$. We split the regret into pulls of arms with gap at most $\Delta_{k_0}$ and pulls of arms
with larger gap. Since every pull of an arm with gap at most $\Delta_{k_0}$ contributes at most $\Delta_{k_0}$, we decompose the regret as:
\begin{align*}
    \mathcal{R}_{\vect{x}_g}(T)
    = \sum_{t=\tau}^T \Delta_{\vect{x}_t}
      \mathbf{1}\curlybracket{t\in\mathcal{T}_{\mathrm{MOSS}}}
      \mathbf{1} \curlybracket{\Delta_{\vect{x}_t} \leq \Delta_{k_0}}
    + \sum_{t=\tau}^T \Delta_{\vect{x}_t}
      \mathbf{1}\curlybracket{t\in\mathcal{T}_{\mathrm{MOSS}}}
      \mathbf{1} \curlybracket{\Delta_{\vect{x}_t} > \Delta_{k_0}}
\end{align*}
For an arm with gap $\Delta_k \geq \Delta_{k_0}$, we split its per-pull cost into a
baseline part $\Delta_{k_0}$ and an excess part $\Delta_k - \Delta_{k_0}$. The
baseline part is paid at most once per round, so its total over the horizon is
at most $T\Delta_{k_0}$, which yields
\begin{align*}
    \mathcal{R}_{\vect{x}_g}(T) \le T\Delta_{k_0} + \sum_{k>k_0}(\Delta_k-\Delta_{k_0})N_k(T),
\end{align*}
where $N_k(T)$ is the number of exploitation pulls of arm $k$ between $\tau$ and
$T$. It remains to bound the expectation of the second term,
$\E[\sum_{k>k_0}(\Delta_k-\Delta_{k_0})N_k(T)]$.

\paragraph{Step 3: decomposition into underestimation and overestimation events.}
For $k>k_0$, define the intermediate threshold
\begin{align*}
    z_k := \mu_{\vect{x}_g}-\frac{\Delta_k}{2}
    = \frac{\mu_{\vect{x}_g}+\mu_k}{2}
\end{align*}
and let $z_{k_0}=+\infty$. Thus $z_k$ lies halfway between the mean of the reference arm $\vect{x}_g$ and the mean of arm $k$. At any exploitation round $t\ge\tau$, the reference arm $\vect{x}_g$ is active and has been observed, hence eligible for selection in the argmax. If the policy instead selects a suboptimal arm $k>k_0$, its index must be at least that of $\vect{x}_g$:
\begin{align*}
    B_t(k)\ge B_t(\vect{x}_g).
\end{align*}
Such a pull can be charged to one of two events:
\begin{align*}
    \text{(A) the reference arm $\vect{x}_g$ is underestimated} \quad \text{i.e., } \quad B_t(\vect{x}_g) \leq z_k
\end{align*}
or
\begin{align*}
    \text{(B) the selected arm $k$ is overestimated} \quad \text{i.e., } \quad B_t(k) \geq z_k
\end{align*}
Indeed, if $(A)$ fails, then $B_t(\vect{x}_g)\ge z_k$, and hence $B_t(k)\ge B_t(\vect{x}_g)\ge z_k$, so $(B)$ holds; thus one of the two events always occurs. We now adapt this decomposition to the growing-active-set index. Let the lower-envelope index of the reference arm be defined as
\begin{align*}
    B_t^{-}(\vect{x}_g)
    := \widehat\mu_{\vect{x}_g}(n_{\vect{x}_g}(t))
    + b_t^{-}(n_{\vect{x}_g}(t))
\end{align*}
Since $b_t^{-}\le b_t$, if $(A)$ holds, then 
\begin{align*}
    B_t^{-}(\vect{x}_g) \leq B_t(\vect{x}_g) \leq z_k
\end{align*}
Thus underestimation of $\vect{x}_g$ can be controlled through the lower-envelope event. Similarly, let the upper-envelope index of arm $k$ be defined as
\begin{align*}
    B_t^{+}(k) := \widehat\mu_k(n_k(t)) + b_T^{+}(n_k(t))
\end{align*}
Since $b_t\le b_T^{+}$, if $(B)$ holds, then 
\begin{align*}
    B_t^{+}(k) \geq B_t(k) \geq z_k
\end{align*}
Thus, overestimation of arm $k$ can be controlled through the upper-envelope event. 

To bound the regret, we use the usual peeling over gap thresholds. Suppose arm $i>k_0$ is selected. Its excess regret above the baseline $\Delta_{k_0}$ is
\begin{align*}
    \Delta_i-\Delta_{k_0}
    =
    \sum_{k=k_0+1}^{i}
    (\Delta_k-\Delta_{k-1})
\end{align*}
Hence a pull of arm $i$ is charged across the thresholds $k=k_0+1,\ldots,i$. For a fixed threshold $k\le i$, suppose that $B_t^{-}(\vect{x}_g)<z_k$, i.e., that the lower-envelope index of the reference arm $\vect{x}_g$ is below the threshold $z_k$. We charge the threshold increment $\Delta_k-\Delta_{k-1}$ to an underestimation event of $\vect{x}_g$. Otherwise, if $B_t^{-}(\vect{x}_g)\ge z_k$, then since the algorithm selected arm $i$ we have $B_t(i)\ge B_t(\vect{x}_g)$, and it follows that
\begin{align*}
    B_t(i)\ge B_t(\vect{x}_g)\ge B_t^{-}(\vect{x}_g)\ge z_k
\end{align*}
Since $i\ge k$ and the gaps are in nondecreasing order, we have $\Delta_i\ge \Delta_k$, and the corresponding midpoint thresholds satisfy the reverse inequality:
\begin{align*}
    z_i = \mu_{\vect{x}_g}-\frac{\Delta_i}{2}
    \leq \mu_{\vect{x}_g}-\frac{\Delta_k}{2} = z_k
\end{align*}
Hence
\begin{align*}
    B_t(i)\ge z_k\ge z_i
\end{align*}
Since the upper-envelope index dominates the actual index of arm $i$, i.e., $B_t(i)\le B_t^{+}(i)$, we have
\begin{align*}
    B_t^{+}(i)\ge B_t(i)\ge z_i
\end{align*}
Thus, for every threshold \(k\le i\), the corresponding charge
\(\Delta_k-\Delta_{k-1}\) is accounted for in one of two ways:
either the lower-envelope index of \(\vect{x}_g\) falls below \(z_k\), or the selected
arm \(i\) has upper-envelope index above its own midpoint \(z_i\). Therefore
all charges from selecting arm \(i\) are covered by the underestimation events
of \(\vect{x}_g\) and the overestimation events of selected suboptimal arms.
    
For $r\ge0$ and $k>k_0$, define the underestimation event
\begin{align}
    A_r^{\vect{x}_g}(k)
    := \parenthese{\exists s\ge1:
    \widehat\mu_{\vect{x}_g,s} + b_{r+s}^{-}(s) < z_k}.
\end{align}
Here $r$ counts the pulls of large-gap arms (gap larger than $\Delta_{k_0}$)
that have occurred so far, and $s$ counts the pulls of the reference arm. After
$r$ such pulls and $s$ pulls of $\vect{x}_g$, the calendar time is at least
$r+s$. Since the bonus $b_t^-(s)$ is nondecreasing in $t$, evaluating it at the
lower bound $r+s$ can only shrink it, so the event above contains the true
underestimation event at the current time; bounding $A_r^{\vect{x}_g}(k)$
therefore suffices.

For $r\ge1$ and $i>k_0$, define the overestimation event
\begin{align}
    C_r^i := \parenthese{\widehat\mu_{i,r} + b_T^{+}(r) \ge z_i}
\end{align}
where $r$ is the number of times arm $i$ has been pulled. By the charging argument, we have
\begin{align*}
\E\bracket{\mathcal{R}_{\vect{x}_g}(T)}
\le T\Delta_{k_0}
+ U + V
\end{align*}
where
\begin{align*}
        U := \sum_{k=k_0+1}^{K} (\Delta_k-\Delta_{k-1}) \sum_{r\ge0} \mathbb P(A_r^{\vect{x}_g}(k))
\end{align*}
and
\begin{align*}
        V := \sum_{i=k_0+1}^{K} \Delta_i \sum_{r\ge1} \mathbb P(C_r^i)
\end{align*}
The term $U$ controls the threshold contributions charged to underestimation of the reference arm $\vect{x}_g$, while $V$ controls the pulls charged to overestimation of suboptimal arms. In $V$ we charge the full gap $\Delta_i$ to each overestimation pull rather than the per-threshold increment; this only inflates the bound, since the increments over $k_0<k\le i$ sum to at most $\Delta_i$, and it keeps the overestimation term simpler.

\begin{claim}[Lemma 4, \cite{degenne2016anytime}]
\label{claim:moss-anytime-peeling-bounds}
Let $\widetilde r_k := \lceil K/(2\Delta_k^2)\rceil$. There exists a constant
$c_\alpha>0$, depending only on $\alpha$, such that, for every $k>k_0$,
\begin{align}
\sum_{r>\widetilde r_k}\mathbb P(A_r^{\vect{x}_g}(k))
\leq c_\alpha \frac{K}{\Delta_k^2} \bracket{\overline{\log}\parenthese{\frac{2eT\Delta_k^2}{K}}+1}
\end{align}
\end{claim}

In our setting, $b_t^{-}$ coincides with the fixed-$K$ anytime MOSS confidence
radius with $K=K_T$, so the claim applies to the lower-envelope underestimation
events $A_r^{\vect{x}_g}(k)$ verbatim.

\paragraph{Step 4: peeling of the two probability sums.}
Recall $\widetilde r_k := \lceil K/(2\Delta_k^2)\rceil$ from
Claim~\ref{claim:moss-anytime-peeling-bounds}, and for every $k>k_0$ define
\begin{align*}
    \widetilde r'_k := \left\lceil\frac{c'_\alpha \overline{\log}\parenthese{\frac{2\lambda T\Delta_k^2}{K}}}{\Delta_k^2}\right\rceil
\end{align*}
where $c'_\alpha>0$ is a sufficiently large constant depending only on $\alpha$.
We bound $U+V$ by $A+B+C+D$, where
\begin{align*}
    A &:= \sum_{k=k_0+1}^{K} (\Delta_k-\Delta_{k-1})\,\widetilde r_k,
    &\quad
    B &:= \sum_{k=k_0+1}^{K} (\Delta_k-\Delta_{k-1})
          \sum_{r>\widetilde r_k}\mathbb P(A_r^{\vect{x}_g}(k)),
    \\[2pt]
    C &:= \sum_{k=k_0+1}^{K} \Delta_k\,\widetilde r'_k,
    &\quad
    D &:= \sum_{k=k_0+1}^{K} \Delta_k
          \sum_{r>\widetilde r'_k}\mathbb P(C_r^k).
\end{align*}
We now bound these four terms.

\paragraph{Bound on $A$.}
Since $\widetilde r_k \le K/(2\Delta_k^2)+1$, we have
\begin{align*}
    A
    \le
    \sum_{k=k_0+1}^{K}
    (\Delta_k-\Delta_{k-1})
    \parenthese{\frac{K}{2\Delta_k^2}+1}.
\end{align*}
By monotonicity of the gaps and a sum-integral comparison,
$\sum_{k>k_0}(\Delta_k-\Delta_{k-1})/\Delta_k^2 \le 2/\Delta_{k_0+1}$, while
$\sum_{k>k_0}(\Delta_k-\Delta_{k-1}) \le \Delta_K \le 1$ because rewards are
supported on $[0,1]$. Therefore $A \le K/\Delta_{k_0+1}+1$.

\paragraph{Bound on $B$.}
The term $B$ is the underestimation term for the reference arm $\vect{x}_g$.
Since $b_t^-$ is exactly the fixed-$K$ MOSS confidence radius with $K=K_T$,
Claim~\ref{claim:moss-anytime-peeling-bounds} bounds
$\sum_{r>\widetilde r_k}\mathbb P(A_r^{\vect{x}_g}(k))$ by
$c_\alpha (K/\Delta_k^2)[\overline{\log}(2eT\Delta_k^2/K)+1]$, and since
$\lambda\ge1$ we may replace $2eT\Delta_k^2/K$ by $2e\lambda T\Delta_k^2/K$
inside the logarithm. Hence
\begin{align*}
    B
    \leq
    c_\alpha K
    \sum_{k=k_0+1}^{K}
    (\Delta_k-\Delta_{k-1})
    \frac{
    \overline{\log}\parenthese{
    \frac{2e\lambda T\Delta_k^2}{K}
    }+1
    }{\Delta_k^2}.
\end{align*}
A sum-integral comparison bounds the sum by
$c\,[\overline{\log}(2e\lambda T\Delta_{k_0+1}^2/K)+1]/\Delta_{k_0+1}$ for a
universal constant $c>0$, whence
\begin{align*}
    B
    \leq
    c_\alpha
    \frac{K}{\Delta_{k_0+1}}
    \bracket{
        \overline{\log}\parenthese{
            \frac{2e\lambda T\Delta_{k_0+1}^2}{K}
        }+1
    }.
\end{align*}
    
\paragraph{Bound on $C$.}
By definition of $\widetilde r'_k$, and splitting off the $+1$,
\begin{align*}
C
\le
c'_\alpha
\sum_{k=k_0+1}^{K}
\frac{
\overline{\log}\parenthese{
\frac{2\lambda T\Delta_k^2}{K}
}
}{\Delta_k}
+
\sum_{k=k_0+1}^{K}\Delta_k.
\end{align*}
Since each $\Delta_k\le 1$ and the sum has at most $K$ terms, the second sum is at most $K$. For
the first, write $\delta:=\Delta_{k_0+1}$ and $r=\Delta_k/\delta\ge1$; by the inequality
$\sup_{r\ge1}\overline{\log}(Ar^2)/r \le c\,(\overline{\log}(A)+1)$ (valid for a
universal constant $c$ and all $A>0$), applied with $A=2\lambda T\delta^2/K$, each summand
is at most $(c_\lambda/\delta)[\overline{\log}(2\lambda T\delta^2/K)+1]$ for a
constant $c_\lambda$ depending only on $\lambda$. Therefore,
\begin{align*}
C
\le
c_{\alpha,\lambda}
\frac{K}{\Delta_{k_0+1}}
\bracket{
\overline{\log}\parenthese{
\frac{2\lambda T\Delta_{k_0+1}^2}{K}
}+1
} + K.
\end{align*}

\paragraph{Bound on $D$.}
The term $D$ is the overestimation tail for suboptimal arms. By the definition
of $\widetilde r'_k$, for all $r>\widetilde r'_k$ we have $b_T^+(r)\le\Delta_k/4$,
so on the event $C_r^k$ the inequality
$\widehat\mu_{k,r}+b_T^+(r)\ge z_k = \mu_k+\Delta_k/2$ forces
$\widehat\mu_{k,r}-\mu_k \ge \Delta_k/4$. Since rewards are supported on $[0,1]$,
Hoeffding's inequality gives $\mathbb P(C_r^k)\le \exp(-r\Delta_k^2/8)$.
Consequently,
\begin{align*}
\sum_{r>\widetilde r'_k}\mathbb P(C_r^k)
\le
\sum_{r\ge 1}\exp\parenthese{-\frac{r\Delta_k^2}{8}}
\le
\frac{c}{\Delta_k^2}
\end{align*}
for a universal constant $c>0$, and therefore
$D \le c\sum_{k>k_0}1/\Delta_k \le cK/\Delta_{k_0+1}$.
Combining the bounds on $A,B,C,D$, we obtain
\begin{align*}
\E\bracket{\mathcal{R}_{\vect{x}_g}(T)}
\le
T\Delta_{k_0}
+
C_{\alpha,\lambda}
\frac{K}{\Delta_{k_0+1}}
\bracket{
\overline{\log}\parenthese{
\frac{2e\lambda T\Delta_{k_0+1}^2}{K}
}+1
}
+
K
\end{align*}
where $C_{\alpha,\lambda}$ depends on $\alpha$ and $\lambda$.

\paragraph{Step 5: distribution-free conversion.}
We now convert the gap-dependent bound into a distribution-free bound, with the
gaps ordered as $0=\Delta_1\le\cdots\le\Delta_K$. If $K>T$, the trivial bound
$\mathcal R_{\vect{x}_g}(T)\le T \le \sqrt{KT}$ already suffices, so we assume
$K\le T$ and set $\varepsilon := \sqrt{K/T}\le1$. We distinguish two cases. 

\emph{Case 1: all gaps are small.}
If $\Delta_k\le\varepsilon$ for all $k$, then every exploitation round has regret
at most $\varepsilon$, so $\mathcal R_{\vect{x}_g}(T)\le T\varepsilon=\sqrt{KT}$.

\emph{Case 2: at least one gap is larger than \(\varepsilon\).}
Since $(\Delta_k)$ is nondecreasing with $\Delta_1=0\le\varepsilon$ and some gap
exceeds $\varepsilon$, there is a last index with gap at most $\varepsilon$; set
$k_0 := \max\{k:\Delta_k\le \varepsilon\}$, so that $1\le k_0<K$ and
$\Delta_{k_0}\le \varepsilon < \Delta_{k_0+1}$. The baseline term then satisfies
$T\Delta_{k_0}\le T\varepsilon=\sqrt{KT}$. Writing $r:=\Delta_{k_0+1}/\varepsilon>1$,
so that $\Delta_{k_0+1}=r\sqrt{K/T}$, we get
$K/\Delta_{k_0+1}=\sqrt{KT}/r$ and $2e\lambda T\Delta_{k_0+1}^2/K=2e\lambda r^2$,
hence
\begin{align*}
\frac{K}{\Delta_{k_0+1}}
\bracket{
\overline{\log}\parenthese{
\frac{2e\lambda T\Delta_{k_0+1}^2}{K}
}+1
}
=
\sqrt{KT}\,
\frac{
\overline{\log}(2e\lambda r^2)+1
}{r}
\le
H_\lambda\sqrt{KT},
\end{align*}
where $H_\lambda:=\sup_{r\ge1}[\overline{\log}(2e\lambda r^2)+1]/r$ is finite (the
numerator grows logarithmically in $r$, the denominator linearly). Plugging this
$k_0$ into the gap-dependent bound gives
$\E[\mathcal{R}_{\vect{x}_g}(T)]\le C_{\alpha,\lambda}\sqrt{KT}+K$, and since
$K\le T$ implies $K\le\sqrt{KT}$, we obtain
$\E[\mathcal{R}_{\vect{x}_g}(T)]\le C'_{\alpha,\lambda}\sqrt{KT}$. Recalling $K=K_T$,
we obtain
\begin{align*}
\E
\bracket{
\sum_{t=\tau}^{T}\Delta^{\vect{x}_g}_{\vect{x}_t}
\mathbf{1}\curlybracket{t \in \mathcal{T}_{\mathrm{MOSS}}}
}
\le
C'_{\alpha,\lambda}\sqrt{K_TT}.
\end{align*}
\end{proof}
%
%
%
%
\begin{theorem}[Restatement of \Cref{thm:quantile-regret}]
Let $\rho \in (0,1)$. We assume that rewards are supported on $[0,1]$ and that each proposed configuration lies in $\mathrm{TOP}_\rho$ with conditional probability at least $p_\rho > 0$:
\begin{equation*}
\forall t \ge 1,~\mathbb P_{P_{\mathcal{O}}}(\vect{x}_t \in \mathrm{TOP}_\rho \mid \mathcal{F}_{t-1}) \ge p_\rho.
\end{equation*}
Then the expected cumulative $\rho$-regret of IMOSS satisfies
\begin{align*}
    \mathbb E\bracket{\mathcal R_T^{\rho,+}}
    =
    O\!\parenthese{
        p_\rho^{-1/\beta}
        +
        T^{(1+\beta)/2}
    }.
\end{align*}
\end{theorem}

\begin{proof}
Let $E_t$ be the event that no top-$\rho$ configuration has entered the active set by the end of round $t$:
\begin{align*}
E_t := \curlybracket{\mathcal{M}_t \cap \mathrm{TOP}_\rho = \varnothing}
\end{align*}
and let
\begin{align*}
\tau_\rho := \inf\curlybracket{t : \mathcal{M}_t \cap \mathrm{TOP}_\rho \neq \varnothing}
\end{align*}
be the first round at which a top-$\rho$ configuration is active.
We decompose the clipped $\rho$-regret into three parts:
\begin{align*}
\mathcal{R}_T^{\rho,+}
\le
\mathcal{R}_T^{\mathrm{disc}}
+
\mathcal{R}_T^{\mathrm{expand}}
+
\mathcal{R}_T^{\mathrm{MOSS}}
\end{align*}
where $\mathcal{R}_T^{\mathrm{disc}}$ accounts for rounds before a top-$\rho$ arm has been discovered,
$\mathcal{R}_T^{\mathrm{expand}}$ accounts for expansion pulls after discovery, and $\mathcal{R}_T^{\mathrm{MOSS}}$ accounts for exploitation pulls
once a top-$\rho$ arm is active.

\noindent\textbf{Discovery cost.}
Before discovery every served configuration has mean below $\mu_\rho$, so each
round contributes at most $1$ and
\begin{align*}
\mathcal{R}_T^{\mathrm{disc}}
:=
\sum_{t=1}^{T}
\parenthese{\mu_\rho-\mu_{\vect{x}_t}}_+
\mathbf{1}\curlybracket{t<\tau_\rho}
\le
\sum_{t=1}^{T}\mathbf{1}\curlybracket{t<\tau_\rho}
=
\sum_{t=1}^{T}\mathbf{1}(E_t),
\end{align*}
since $\{t<\tau_\rho\}=\{\mathcal{M}_t\cap\mathrm{TOP}_\rho=\varnothing\}=E_t$.
By the coverage assumption, each of the $K_t$ arms admitted up to round $t$ lies
in $\mathrm{TOP}_\rho$ with conditional probability at least $p_\rho$ given the
history before its admission. Iterating this over the $K_t$ admissions,
\begin{align*}
\mathbb P(E_t) \leq \parenthese{1-p_\rho}^{K_t} \leq \exp(-p_\rho K_t),
\end{align*}
where the last step uses $1-x \leq e^{-x}$ for all $x \geq 0$.
Since $K_t=\Omega(t^\beta)$, there is a constant $c_0>0$ with $K_t\ge c_0 t^\beta$
for all $t\ge1$, hence $\mathbb P(E_t) \leq \exp(-p_\rho c_0 t^\beta)$. Taking
expectations in the display above,
\begin{align*}
    \E \bracket{\mathcal{R}_T^{\mathrm{disc}}}
    \leq \sum_{t=1}^T \mathbb P(E_t)
    \leq \sum_{t=1}^{\infty} \exp(-p_\rho c_0 t^\beta).
\end{align*}
The sum is bounded by an integral:
\begin{align*}
\sum_{t=1}^{\infty} e^{-p_\rho c_0 t^\beta}
\le
1+\int_0^\infty e^{-p_\rho c_0 x^\beta}\,dx
\end{align*}
With the change of variables $u=p_\rho c_0 x^\beta$, we have
\begin{align*}
\int_0^\infty e^{-p_\rho c_0 x^\beta}\,dx
=
\frac{\Gamma(1/\beta)}{\beta}\,(p_\rho c_0)^{-1/\beta}
=
O\!\parenthese{p_\rho^{-1/\beta}}
\end{align*}
Therefore,
\begin{align*}
\E\bracket{\mathcal{R}_T^{\mathrm{disc}}}
\le
1+\frac{\Gamma(1/\beta)}{\beta}\,(p_\rho c_0)^{-1/\beta}
=
O(p_\rho^{-1/\beta})
\end{align*}

\noindent\textbf{Expansion cost after discovery.}
Recall that $\mathcal{T}_{\mathrm{expand}}$ is the set of expansion rounds, and define
\begin{align*}
\mathcal{R}_T^{\mathrm{expand}}
:=
\sum_{t=\tau_\rho}^{T}
\parenthese{\mu_\rho-\mu_{\vect{x}_t}}_+
\mathbf{1}\curlybracket{t\in\mathcal{T}_{\mathrm{expand}}}.
\end{align*}
At every expansion round, the newly admitted configuration is pulled once and
may incur regret at most $1$. The number of expansion rounds by time $T$ is at
most $K_T$, and $K_T=O(T^\beta)$. Therefore,
\begin{align*}
\E\bracket{\mathcal{R}_T^{\mathrm{expand}}}
\leq K_T
= O(T^\beta).
\end{align*}

\noindent\textbf{Exploitation cost after discovery.}
If $\tau_\rho>T$, no top-$\rho$ arm is discovered within the horizon; the sum
defining $\mathcal{R}_T^{\mathrm{MOSS}}$ below is then empty and
$\mathcal{R}_T^{\mathrm{MOSS}}=0$, so we may assume $\tau_\rho\le T$.
Let $\vect{x}_g$ be the first top-$\rho$ arm admitted to the active set, i.e., the arm admitted at round $\tau_\rho$.
Then $\mu_{\vect{x}_g} \ge \mu_\rho$, so for every $t$,
\begin{align*}
\parenthese{\mu_\rho - \mu_{\vect{x}_t}}_+ \leq \parenthese{\mu_{\vect{x}_g} - \mu_{\vect{x}_t}}_+
= \Delta_{\vect{x}_t}^{\vect{x}_g}.
\end{align*}
Summing over exploitation rounds after discovery,
\begin{align*}
\mathcal{R}_T^{\mathrm{MOSS}}
:=
\sum_{t=\tau_\rho}^{T}
\parenthese{\mu_\rho-\mu_{\vect{x}_t}}_+
\mathbf{1}\curlybracket{t \in \mathcal{T}_{\mathrm{MOSS}}}
\le
\sum_{t=\tau_\rho}^{T}
\Delta_{\vect{x}_t}^{\vect{x}_g}
\mathbf{1}\curlybracket{t \in \mathcal{T}_{\mathrm{MOSS}}}.
\end{align*}
Since $\mathcal{M}_t$ and $\mathrm{TOP}_\rho$ are determined by
the admissions and the fixed means, $\tau_\rho$ is an $(\mathcal{F}_t)$-stopping
time, and $\vect{x}_g$ is the arm admitted at round $\tau_\rho$, hence
$\mathcal{F}_{\tau_\rho-1}$-measurable and active for all $t\ge\tau_\rho$.
We may now invoke Lemma~\ref{lem:growing-moss-allocation} with reference arm
$\vect{x}_g$ and stopping time $\tau=\tau_\rho$: the schedule $K_t=\Theta(t^\beta)$
satisfies the regular growth condition. Indeed, using $K_t=\Omega(t^\beta)$ to
upper-bound $t/K_t$ and $K_T=O(T^\beta)$ to lower-bound $T/K_T$,
\begin{align*}
\frac{t}{K_t} = O\parenthese{t^{1-\beta}} = O\parenthese{T^{1-\beta}}
= O\parenthese{\frac{T}{K_T}}
\qquad \forall t\le T,
\end{align*}
where we used that $t^{1-\beta}$ is nondecreasing in $t$ for $\beta<1$. Hence
$t/K_t \le \lambda\, T/K_T$ for a constant $\lambda\ge1$ depending only on the
schedule. Lemma~\ref{lem:growing-moss-allocation} therefore gives
\begin{align*}
\E\bracket{\mathcal{R}_T^{\mathrm{MOSS}}}
\le
\E\bracket{\sum_{t=\tau_\rho}^{T}
\Delta_{\vect{x}_t}^{\vect{x}_g}
\mathbf{1}\curlybracket{t \in \mathcal{T}_{\mathrm{MOSS}}}}
\le
O \parenthese{\sqrt{K_TT}}.
\end{align*}
Using $K_T = O\parenthese{T^\beta}$, we have
\begin{align*}
\E\bracket{\mathcal{R}_T^{\mathrm{MOSS}}}
= O \parenthese{T^{(1+\beta)/2}}
\end{align*}
Therefore, combining the three parts, we have
\begin{align*}
    \E\bracket{\mathcal{R}_T^{\rho,+}}
    \leq O\parenthese{p_\rho^{-1/\beta}}
    + O(T^\beta)
    + O\parenthese{T^{(1+\beta)/2}}
\end{align*}
Since $\beta \in (0,1)$, we have $T^\beta \leq T^{(1+\beta)/2}$. Therefore,
\begin{align*}
    \E\bracket{\mathcal{R}_T^{\rho,+}}
    \leq O\parenthese{p_\rho^{-1/\beta}}
    + O\parenthese{T^{(1+\beta)/2}}
\end{align*}
\end{proof}

\section{Complexity Analysis}
\label{sec:complexity}

We separate the complexity of the IMOSS policy, which treats the oracle as a black box, from the complexity of a single oracle call, which depends on the instantiation. Throughout, $K_t = |\mathcal{M}_t|$ denotes the size of the active set, $d$ the number of parameters of $\xcal$, and $c$ the size of the candidate pool drawn by the oracle.

\paragraph{IMOSS.}
After $T$ rounds the active set holds $K_T = O(T^{\beta})$ arms, each stored as its configuration together with a running mean and two counters, so memory grows as $O(d\,T^{\beta})$, sublinear in the horizon. At each round, deciding whether to expand the active set amounts to comparing $K_t$ with $t^{\beta}$, which takes constant time. An exploitation round runs \Cref{alg:moss}, which evaluates the index~\eqref{eq:moss} of every arm and returns the maximizer; since the statistics entering the index are maintained incrementally, each evaluation is $O(1)$ and the full pass costs $O(K_t) = O(T^{\beta})$. An expansion round replaces this pass with a single oracle call, and since the oracle is queried only when an arm is admitted, at most $K_T = O(T^{\beta})$ calls occur over the whole horizon. Writing $C_{\mathcal{O}}(K)$ for the cost of one oracle call given an active set of $K$ arms, the total running time after $T$ rounds is
\begin{equation*}
    O\parenthese{T^{1+\beta} + T^{\beta}\, C_{\mathcal{O}}(T^{\beta})},
\end{equation*}
where the first term accounts for the index passes and the second for the oracle calls. For IMOSS-Random, one call draws a single configuration in $C_{\mathcal{O}}(K) = O(d)$, so the total is $O(T^{1+\beta})$.

\paragraph{TPE oracle.}
One call to the TPE oracle sorts the active set by the index~\eqref{eq:moss} in $O(K \log K)$, fits the Parzen densities $\ell$ and $g$ with one kernel per arm and per parameter in $O(d\,K)$, samples $c$ candidates from $\ell$ in $O(c\,d)$, and scores each candidate by the ratio $\ell(\vect{x})/g(\vect{x})$, each density evaluation summing over $O(K)$ kernels per parameter, for $O(c\,d\,K)$ in total. Hence
\begin{equation*}
    C_{\mathrm{TPE}}(K) = O\parenthese{K \parenthese{\log K + c\,d}},
\end{equation*}
and the oracle contributes $O\parenthese{T^{2\beta}(\log T + c\,d)}$ to the total running time. Since $2\beta < 1+\beta$ for $\beta \in (0,1)$, this is asymptotically dominated by the index passes, and IMOSS-TPE runs in $O(T^{1+\beta})$ total time for fixed $c$ and $d$.

\paragraph{Mutate-KL$\times$PE oracle.}
One call selects the best arm so far $\hat{\vect{x}}$ and refreshes the coordinate bandit in $O(K)$, evaluates the $d$ KL-UCB indices in $O(d)$ with a constant number of bisection steps each, and sorts the active set by the index~\eqref{eq:moss} in $O(K \log K)$. It then fits and scores the Parzen pair for the \emph{selected coordinate only}: $O(K)$ kernels to fit, $c$ candidate draws, and $O(K)$ kernel evaluations per candidate, for $O(c\,K)$. Hence
\begin{equation*}
    C_{\mathrm{KL\times PE}}(K) = O\parenthese{K \parenthese{\log K + c} + d},
\end{equation*}
a factor $d$ cheaper in the density work than the TPE oracle, which fits and evaluates one kernel per parameter. The oracle contributes $O\parenthese{T^{2\beta}(\log T + c)}$ to the total running time and, like the TPE oracle, is asymptotically dominated by the index passes.

\paragraph{TabPFN oracle.}
One call to the TabPFN oracle assembles a context table of the $K$ arms and their mean rewards in $O(d\,K)$, then runs an ensemble of $E$ forward passes ($E = 4$ in our implementation) of the pretrained transformer over the table extended with the $c$ candidate rows. With self-attention across the rows of the table, one forward pass costs $O\parenthese{(K+c)^2}$, treating the width and depth of the network as constants, so
\begin{equation*}
    C_{\mathrm{TabPFN}}(K) = O\parenthese{E\,(K+c)^2}.
\end{equation*}
Once $K$ exceeds the pool size $c$, the oracle contributes $O(E\,T^{3\beta})$ to the total running time. For $\beta < 1/2$ this remains dominated by the $O(T^{1+\beta})$ index passes; at $\beta = 1/2$, the value used in our experiments, both terms are $O(T^{3/2})$, up to model-dependent constants; for larger $\beta$ the oracle calls dominate. In practice the wall-clock cost of a call is governed by the constant factors of the pretrained network rather than by the cost's asymptotic growth in $K$, since each call is a single batched forward pass.


\section{Oracle ablation: which oracle, and when?}
\label{sec:oracle-ablation}

We have already introduced three learned admission oracles. A natural question to ask is \emph{which oracle should be used on which problem}. We compare them,
together with the uniform baseline, under the same active-set schedule and
MOSS allocation rule, with $\beta=0.5$ and $T=5000$, over $10$ seeds, so that
every difference is attributable to the proposal mechanism alone.

We arrange the four oracles so that two of the comparisons are controlled,
each isolating a single design choice. IMOSS-TPE (\Cref{subsec:tpe}) and
IMOSS-mutate-KL$\times$PE (\Cref{subsec:klxtpe}) rank candidates with the
same machinery, a Parzen density ratio $\ell/g$ fitted on the same
MOSS-scored good/bad split, and differ only in \emph{where} the proposal is
drawn: globally, from a density fitted on all good arms, or locally, by
changing one coordinate of the best arm so far. Comparing the two therefore
isolates the effect of \emph{locality}. IMOSS-mutate-KL$\times$PE and
IMOSS-TabPFN (\Cref{subsec:tabfm}) are both local, each scoring a pool of
single-coordinate mutations of the same best arm so far $\hat{\vect{x}}$, and differ only in what ranks
those candidates: a univariate density ratio for the former, a pretrained
reward surrogate for the latter. Comparing the two therefore isolates the
effect of the \emph{surrogate}. IMOSS-Random, which neither localizes nor
learns, is the floor against which they are compared.

We run all four on the three discrete random-forest grids of
\Cref{sec:openml-hpo} and on three LCBench instances defined in \Cref{sec:delayed-feedback-experiment}.
Together the two families give six reward landscapes that differ widely in
structure.

\subsection{Landscape structure and surrogate accuracy}
\label{sec:oracle-ablation-diagnostics}

Two properties govern how much an oracle can help on a given task: how the
reward is spread over the search space, and, for the model-based oracle, how
well that reward can be predicted from the arms tried so far. We measure both
on the random-forest grids before turning to the comparison itself.

\paragraph{What makes the three grids differ.}
The three random-forest grids are the clearest case, since their reward
landscapes can be inspected directly.

\begin{figure}[htbp]
    \centering
    \includegraphics[width=\textwidth]{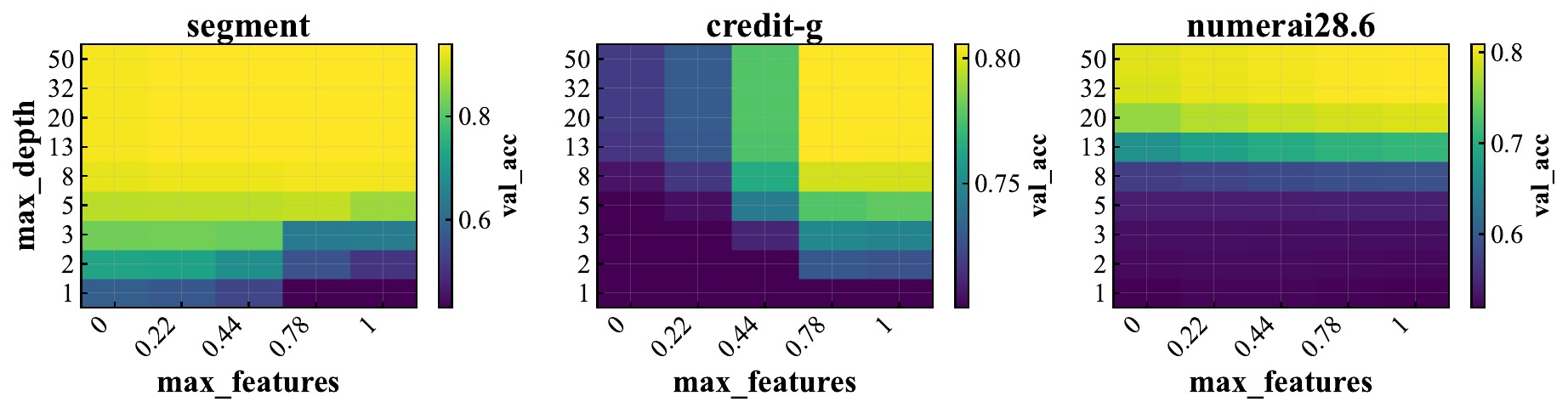}
    \caption{Mean validation accuracy of the three RF tabular benchmarks as a function of \texttt{max\_depth} and \texttt{max\_features}, averaged over \texttt{min\_samples\_leaf} and \texttt{min\_samples\_split}.}
    \label{fig:openml-hpo-landscape-structure}
\end{figure}

\begin{wrapfigure}[18]{r}{0.42\textwidth}
    \centering
    \includegraphics[width=\linewidth]{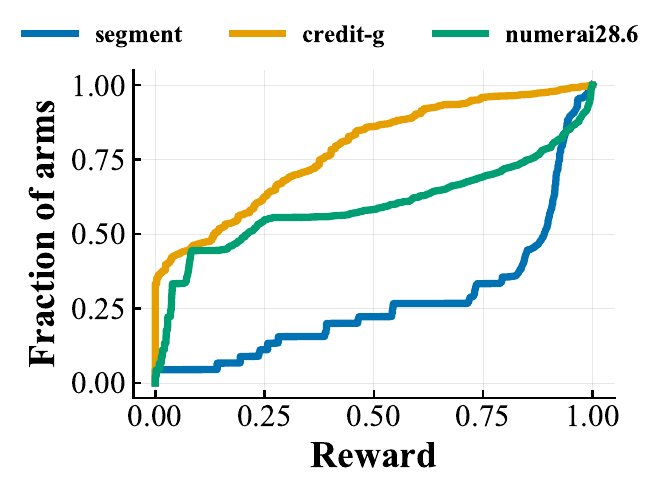}
    \caption{Empirical CDF of arm reward, normalized within each task.}
    \label{fig:openml-hpo-landscape-reward-cdf}
\end{wrapfigure}
\Cref{fig:openml-hpo-landscape-structure} shows that \texttt{segment} and \texttt{numerai28.6} behave alike: reward climbs steadily with \texttt{max\_depth} and barely responds to \texttt{max\_features}, so a single hyperparameter essentially decides how good an arm is. \texttt{credit-g} is the exception: its high-reward arms sit in a vertical band at large \texttt{max\_features} that only emerges once \texttt{max\_depth} is also large, so finding a good arm there means getting two hyperparameters right at once rather than one. In short, on \texttt{segment} and \texttt{numerai28.6} the quality of an arm is driven almost entirely by \texttt{max\_depth}, while on \texttt{credit-g} it depends on \texttt{max\_depth} and \texttt{max\_features} together, with neither one enough on its own.

\Cref{fig:openml-hpo-landscape-reward-cdf} makes the resulting difference in optimum sparsity explicit. The three empirical CDFs of arm reward differ markedly in shape. For \texttt{credit-g} (orange), the CDF rises steeply from $x=0$: arms are concentrated at low reward and only a thin slice near $x=1$ is competitive, indicating a sparse optimum. \texttt{segment} (blue) exhibits the opposite behavior, with the CDF remaining flat across most of the range and rising only near $x=1$; the majority of arms therefore attain near-optimal reward, indicating a broad optimum. \texttt{numerai28.6} (green) lies between these two regimes: its CDF rises faster than \texttt{segment}'s at low reward but far more slowly than \texttt{credit-g}'s. This intermediate profile is consistent with the heatmap: the optimum of \texttt{numerai28.6} is as easy to locate as that of \texttt{segment}, since both are governed primarily by \texttt{max\_depth}, but reward is less uniformly high within the good region, making the arms there harder to rank.

Together, \Cref{fig:openml-hpo-landscape-structure,fig:openml-hpo-landscape-reward-cdf} give the landscape taxonomy we use below to interpret the oracle comparison: \texttt{segment} is a near-one-dimensional, broad-optimum landscape, whereas \texttt{credit-g} has a sparse, interaction-dependent optimum. The location of \texttt{numerai28.6}'s optimum is determined by a single hyperparameter, but rewards are less concentrated within that region.

\paragraph{How well the surrogate models them.}
The surrogate axis of the ablation rests on TabPFN ranking candidates better
than a univariate density ratio does, which it can only do if its predicted
rewards are accurate on these landscapes. We therefore probe its predictions
directly, on the same three grids.

\begin{figure*}[!ht]
    \centering
    \includegraphics[width=\textwidth]{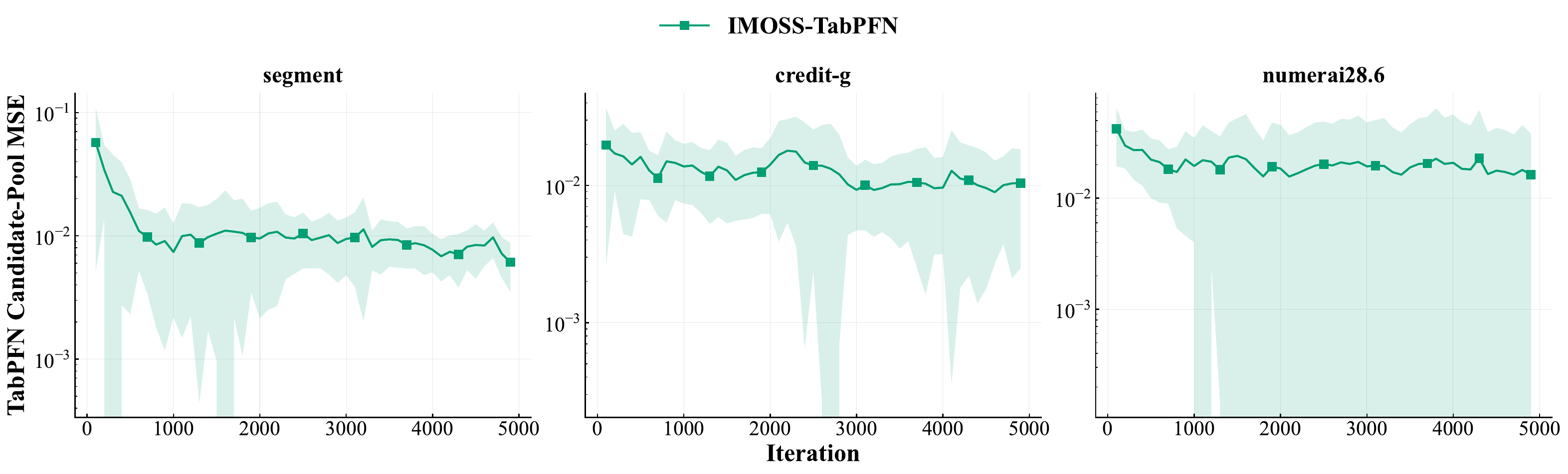}
    \caption{Mean squared error of the TabPFN oracle's reward predictions over its whole candidate pool on the three RF tabular benchmarks (log scale). At every probed exploration step, the oracle scores a pool of $100$ candidate configurations; each point reports the MSE between the predicted and true mean rewards of \emph{all} candidates in that pool, not only the proposed configuration. Curves show means over $30$ paired runs, and shaded bands show one standard deviation across runs.}
    \label{fig:tabpfn-candidate-mse}
\end{figure*}

\Cref{fig:tabpfn-candidate-mse} reports how accurately the TabPFN oracle predicts rewards across its entire candidate pool as the run progresses. On \texttt{segment} and \texttt{credit-g}, the pool-wide MSE drops by roughly an order of magnitude over the first $2000$ iterations as evaluated arms accumulate in the in-context table, showing that the oracle's reward model improves over the whole candidate space and not merely on the configurations it proposes. On \texttt{numerai28.6}, the error stays flat at a low level: most arms have nearly indistinguishable mean rewards, so the pool-wide error is dominated by the residual spread among near-tied arms rather than by ranking improvements, consistent with the landscape analysis above.

\begin{figure*}[!ht]
    \centering
    \includegraphics[width=\textwidth]{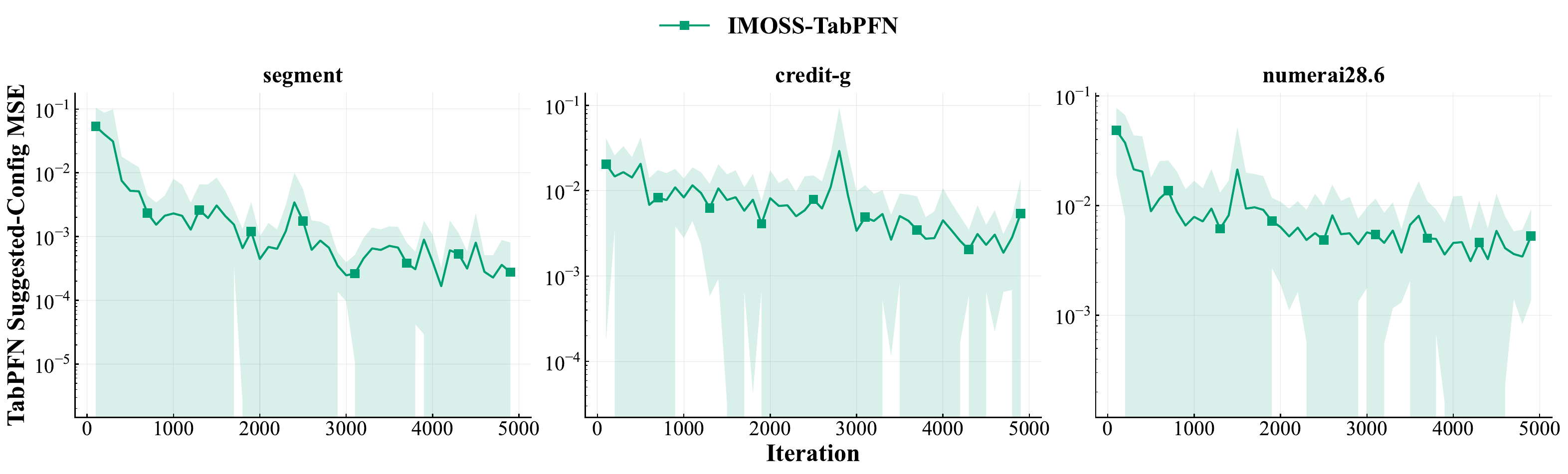}
    \caption{Mean squared error of the TabPFN oracle's reward prediction on the configuration it proposes, on the three RF tabular benchmarks (log scale). At every probed exploration step, each point reports the squared error between the predicted mean reward of the proposed configuration and its true mean reward. Curves show means over $30$ paired runs, and shaded bands show one standard deviation across runs.}
    \label{fig:tabpfn-suggestion-mse}
\end{figure*}

\Cref{fig:tabpfn-suggestion-mse} complements the pool-wide view by measuring accuracy exactly where the oracle acts, on the configuration it proposes. This error falls by one to two orders of magnitude on all three tasks and ends well below the pool-wide MSE, since proposals concentrate in high-reward regions that are densely represented in the in-context table. Notably, it keeps decreasing on \texttt{numerai28.6} even though the pool-wide error stays flat there: the oracle becomes increasingly accurate on the arms it exploits, while the pool-wide error remains dominated by the residual spread of the near-tied bulk.

\subsection{Comparing the oracles}
\label{sec:oracle-ablation-comparison}

\paragraph{A decomposition of oracle quality.}
Let $\mathcal{M}_T$ be the set of arms admitted up to round $T$,
$n_{\vect{x}}$ the number of pulls arm $\vect{x}$ receives, so that
$\sum_{\vect{x}\in\mathcal{M}_T} n_{\vect{x}} = T$.
Recall that $\mathcal{R}_T = \sum_{\vect{x}\in\mathcal{M}_T} n_{\vect{x}}(p^\star - \mu_{\vect{x}})$,
and that $p^\star - \mu_{\vect{x}} = (p^\star - \mu_{\text{best}}) + (\mu_{\text{best}} - \mu_{\vect{x}})$ where $\mu_{\text{best}} = \max_{\vect{x}\in\mathcal{M}_T}\mu_{\vect{x}}$ is the best true mean the oracle admitted.
The online average regret is decomposed as
\begin{equation}
  \label{eq:oracle-decomposition}
  \frac{\mathcal{R}_T}{T} \;=\;
  \underbrace{\bigl(p^\star - \mu_{\text{best}}\bigr)}_{\text{best-arm gap}}
  \;+\;
  \underbrace{\frac{1}{T}\sum_{\vect{x} \in \mathcal{M}_T} n_{\vect{x}}
  \bigl(\mu_{\text{best}} - \mu_{\vect{x}}\bigr)}_{\text{cost of the other arms}}.
\end{equation}
The first term depends only on the oracle: it asks whether the set contains a
good configuration at all, and no allocation rule can reduce it. The second
asks how much is lost on everything else in the set, and is what the policy
pays for every pull it spends away from the best arm it has. An oracle can
therefore lower regret in two ways: by admitting a better best arm, or by
making the rest of the set cheaper to pull. The two controlled comparisons
above separate them.

\begin{figure*}[!ht]
    \centering
    \includegraphics[width=\textwidth]{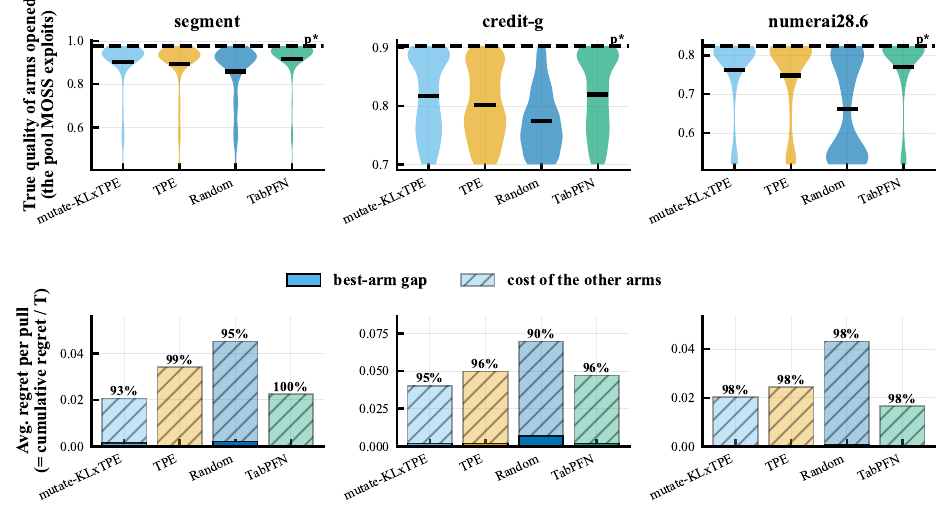}
    \caption{Oracle ablation on the three discrete random-forest grids
    (\Cref{sec:openml-hpo}). \emph{Top:} distribution of the true mean
    rewards of the arms each oracle admits, pooled over $10$ seeds
    (horizontal bar: mean; dashed line: $p^\star$); this is the pool the
    allocation rule must exploit. \emph{Bottom:} average per-pull regret
    split as in \Cref{eq:oracle-decomposition} into the best-arm gap (solid)
    and the cost of the other arms (hatched); the annotation gives the share
    of the second. On the grids every oracle finds a configuration close to
    the optimum, so the best-arm gap is negligible and nearly all the regret
    is what the policy spends on the rest of the set.}
    \label{fig:oracle-decomp-rf}
\end{figure*}

\begin{figure*}[!ht]
    \centering
    \includegraphics[width=\textwidth]{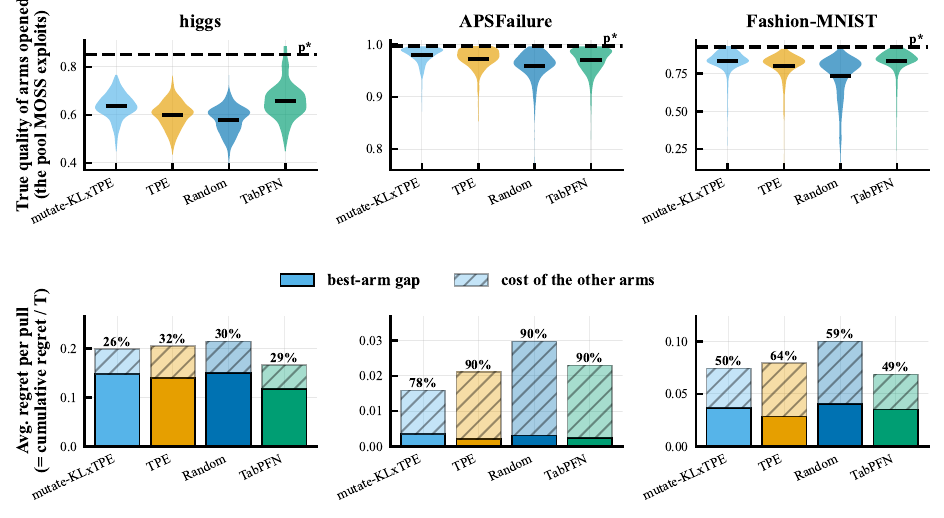}
    \caption{Oracle ablation on the three mixed continuous--discrete LCBench
    instances (\Cref{tab:lcbench-search-space}). Unlike on the grids, the split between the
    two terms varies widely across instances. On \texttt{APSFailure}, the cost
    of the other arms again accounts for nearly all the regret, whereas
    \texttt{higgs} is the one task where the best-arm gap dominates: no oracle
    reaches the optimum, and only the surrogate narrows that gap.}
    \label{fig:oracle-decomp-mixed}
\end{figure*}

\begin{table}[htbp]
    \caption{Cumulative regret $\mathcal{R}_T$ at $T=5000$, mean $\pm$ std
    over $10$ seeds; lower is better. Bold marks the lowest mean per row,
    which need not be a significant advantage: the difference between the two
    local oracles, KL$\times$PE and TabPFN, is significant only on
    \texttt{APSFailure}, in favor of KL$\times$PE. Significance is discussed
    per comparison in the text.}
    \centering
    \small
    \setlength{\tabcolsep}{3.5pt}
    \begin{tabular*}{\linewidth}{l @{\extracolsep{\fill}} llll}
    \toprule
    \textbf{Task} & \textbf{Uniform} & \textbf{TPE} & \textbf{KL$\times$PE} & \textbf{TabPFN} \\
    \midrule
    \multicolumn{5}{l}{\emph{Discrete grids (random forest)}}\\
    \texttt{segment}     & $226 \pm 21$ & $171 \pm 16$ & $\mathbf{103 \pm 37}$ & $113 \pm 26$ \\
    \texttt{credit-g}    & $348 \pm 71$ & $249 \pm 24$ & $\mathbf{201 \pm 56}$ & $235 \pm 112$ \\
    \texttt{numerai28.6} & $216 \pm 31$ & $122 \pm 21$ & $101 \pm 32$ & $\mathbf{83 \pm 20}$ \\
    \addlinespace[0.3em]
    \multicolumn{5}{l}{\emph{Mixed spaces (LCBench)}}\\
    \texttt{higgs}         & $1073 \pm 124$ & $1023 \pm 123$ & $993 \pm 159$ & $\mathbf{829 \pm 306}$ \\
    \texttt{APSFailure}    & $149 \pm 13$ & $105 \pm 20$ & $\mathbf{79 \pm 25}$ & $114 \pm 17$ \\
    \texttt{Fashion-MNIST} & $500 \pm 45$ & $396 \pm 43$ & $371 \pm 153$ & $\mathbf{343 \pm 113}$ \\
    \bottomrule
    \end{tabular*}
    \label{tab:oracle-ablation}
\end{table}

\paragraph{What the comparison shows.}
All three learned oracles accumulate less regret than uniform admission, on
every task. On four of the six tasks almost all the regret is the cost of the
other arms: the set nearly always holds a configuration close to the optimum,
and what is expensive is everything else the policy pulls along the way. On
\texttt{Fashion-MNIST} the two terms are of comparable size, and only on
\texttt{higgs} is it the other way round, with no oracle finding a
configuration near the optimum.

The two local oracles behave similarly: both lower the cost of the other arms below what the global TPE proposal incurs. Neither improves the best-arm gap, and
mutate-KL$\times$PE makes it slightly worse. The two are
statistically indistinguishable except on \texttt{APSFailure}, where
mutate-KL$\times$PE wins.

\paragraph{Why proposing locally helps.}
The cost of the other arms factors into two quantities: how often a pull lands somewhere other than the best arm admitted so far, and how much it costs when it does (\Cref{tab:oracle-mechanism}). The first quantity is close to one on almost every benchmark, so the policy almost never stays on the best arm, and this is more pronounced for the local oracles, whose sets are packed with arms of nearly equal value that MOSS cannot tell apart within the budget. The separation between oracles is therefore carried by the second quantity. 

Every arm a local oracle adds is one parameter away from a good configuration found so far, so the policy wanders among arms that are all nearly as good as the best one, and the cost of wandering is small. A global proposal keeps adding arms from all over the space, where the same wandering is expensive. The local oracle does not make the policy a better chooser, but it makes the choice matter less.

\begin{table}[htbp]
    \caption{Where the cost of the other arms comes from. For each task and
    oracle, the fraction of pulls that do not land on the best arm admitted so
    far, and the mean amount by which such a pull falls short of it. The policy is
    off the best arm for the great majority of pulls under every oracle, and
    most of all under the two local ones, which lose
    the least when it happens.}
    \centering
    \small
    \setlength{\tabcolsep}{3pt}
    \begin{tabular*}{\linewidth}{l @{\extracolsep{\fill}} cccc c cccc}
    \toprule
    & \multicolumn{4}{c}{Pulls not on the best arm}
    & & \multicolumn{4}{c}{Mean loss per such pull} \\
    \cmidrule(lr){2-5}\cmidrule(lr){7-10}
    \textbf{Task} & Unif. & TPE & KL$\times$PE & TabPFN
    & & Unif. & TPE & KL$\times$PE & TabPFN \\
    \midrule
    \multicolumn{10}{l}{\emph{Discrete grids (random forest)}}\\
    \texttt{segment}     & 0.91 & 0.93 & 0.98 & 0.96 & & 0.048 & 0.036 & 0.020 & 0.023 \\
    \texttt{credit-g}    & 0.71 & 0.90 & 0.93 & 0.97 & & 0.097 & 0.057 & 0.042 & 0.046 \\
    \texttt{numerai28.6} & 0.91 & 0.97 & 0.97 & 0.96 & & 0.048 & 0.025 & 0.020 & 0.017 \\
    \addlinespace[0.3em]
    \multicolumn{10}{l}{\emph{Mixed spaces (LCBench)}}\\
    \texttt{higgs}         & 0.79 & 0.84 & 0.99 & 0.99 & & 0.088 & 0.081 & 0.052 & 0.052 \\
    \texttt{APSFailure}    & 0.97 & 0.98 & 0.98 & 0.98 & & 0.027 & 0.019 & 0.013 & 0.021 \\
    \texttt{Fashion-MNIST} & 0.88 & 0.90 & 0.98 & 0.99 & & 0.069 & 0.057 & 0.039 & 0.034 \\
    \bottomrule
    \end{tabular*}
    \label{tab:oracle-mechanism}
\end{table}

The same property also reveals the downside of locality: a one-parameter change cannot take the search far from where it already is. On five of the six tasks, this rarely matters, because good configurations are common and one is always close by. On \texttt{higgs} it does: the good region needs several parameters set correctly at once, so one-at-a-time steps never arrive and the best-arm gap stays large however cheap the other arms become. That is the one task where TabPFN, which ranks candidates with a model of the whole reward surface, admits the best arm of the four oracles.

\paragraph{Practical guidance.}
We recommend mutating the current best configuration, and resorting to a surrogate only when the optimum is expected to be difficult to reach. On most of our tasks the model-free local oracle matches the transformer-based one while fitting a single one-dimensional density per proposal rather than performing a forward pass through a pretrained network. The surrogate justifies its cost only where good configurations are rare and require several parameters to be set correctly at once. The global TPE proposal needs no metric on the space. It models only the density ratio between promising and unpromising configurations rather than the full reward surface, and proposes over a continuum without a pre-declared grid. Mutating the current best arm instead lowers cumulative regret on every task in \Cref{tab:oracle-ablation}, so the global proposal is best read as the metric-free baseline that the mutation oracle improves on.


\section{Additional Experimental Details}
\subsection{HotpotQA: Search Space and Generation Settings}
\label{sec:hotpotqa-details}
The hyperparameter search space, including OpenRouter model identifiers, is
summarized in \cref{tab:hotpotqa-search-space}. Retrieval uses dense embedding
search with \texttt{all-MiniLM-L6-v2} over the full corpus. We use three system prompt templates, shown side by side below.
\texttt{few\_shot} adds examples to the rules; \texttt{zero\_shot}
keeps the rules but drops the examples; and \texttt{naive} keeps only a
one-line instruction.

At inference time, the selected template is sent as the system message. The
user message combines the top-$k$ retrieved passages and the HotpotQA query with the prefixes \texttt{Context:} and \texttt{Question:}, respectively.
Each API call specifies \texttt{model} and \texttt{temperature}; all
other generation parameters (\texttt{max\_tokens}, \texttt{top\_p},
\texttt{seed}, \texttt{stop}) are left unset and default to the provider's
own settings. No provider is pinned on OpenRouter, so requests are served
according to OpenRouter's default routing and fallback policy for the
selected model.
Failed calls are retried up to 20 times with random exponential backoff
(base 2\,s, capped at 300\,s) for rate-limit, provider-overload, and
service-unavailable errors; calls failing with any other error type are
not retried, and the error propagates as a failure. Results for repeated
\texttt{(question, configuration)} pairs are served from a local cache, so
re-running or resuming an experiment does not re-issue identical API
calls.

\begin{table*}[htbp]
    \caption{Hyperparameter search space for the HotpotQA RAG pipeline.
    Generation models are accessed through OpenRouter.}
    \centering
    \begingroup
    \setlength{\tabcolsep}{5pt}
    \renewcommand{\arraystretch}{1.3}
    \begin{tabular}{@{} p{0.18\linewidth} p{0.10\linewidth} p{0.235\linewidth} p{0.365\linewidth} @{}}
    \toprule
    \textbf{Parameter} & \textbf{Type} & \textbf{Domain} & \textbf{OpenRouter ID} \\
    \midrule
    \texttt{top\_k} & Integer & $\{1,\ldots,10\}$ & \\
    \texttt{temperature} & Continuous & $[0, 1]$ & \\
    \texttt{prompt\_template} & Categorical &
    \parbox[t]{\hsize}{\raggedright
        \texttt{few\_shot} \\[0.2em]
        \texttt{zero\_shot} \\[0.2em]
        \texttt{naive}} & \\
    \addlinespace[0.4em]
    \multirow{6}{*}{\parbox[t]{\hsize}{\texttt{model}}} &
    \multirow{6}{*}{\parbox[t]{\hsize}{Categorical}} &
    Qwen 3.5 Flash & \texttt{qwen/qwen3.5-flash-02-23} \\[0.25em]
    & & Ling 2.6 Flash & \texttt{inclusionai/ling-2.6-flash} \\[0.25em]
    & & DeepSeek V4 Flash & \texttt{deepseek/deepseek-v4-flash} \\[0.25em]
    & & Granite 4.1 8B & \texttt{ibm-granite/granite-4.1-8b} \\[0.25em]
    & & HY3 Preview & \texttt{tencent/hy3-preview} \\[0.25em]
    & & Llama 3.2 1B Instruct &
    \parbox[t]{\hsize}{\raggedright\texttt{meta-llama/\allowbreak llama-3.2-\allowbreak 1b-instruct}} \\
    \bottomrule
    \end{tabular}
    \endgroup
    \label{tab:hotpotqa-search-space}
\end{table*}

\medskip
\noindent
\begin{minipage}[t]{0.32\textwidth}
\centering
\textbf{\texttt{few\_shot}}\par\smallskip
\begin{promptbox}[unbreakable]
\begin{PromptVerbatim}[breaklines=true,breaksymbolleft={},breakindent=0pt,breakautoindent=false,fontsize=\scriptsize]
You are a precise question answering assistant. Answer using only the provided context:

Rules:
- Answer in as few words as possible.
- Never explain or add context - just the answer.
- For yes/no questions answer with only "yes" or "no".
- If the context does not contain the answer, output "unknown".

Examples:
Q: What is the capital of France? 
Context: The capital of France is Paris.
A: Paris

Q: What year was the Eiffel Tower built?
Context: The Eiffel Tower was built in 1889.
A: 1889

Q: Is Berlin the capital of Germany? 
Context: The capital of Germany is Berlin.
A: yes
\end{PromptVerbatim}
\end{promptbox}
\end{minipage}\hfill
\begin{minipage}[t]{0.32\textwidth}
\centering
\textbf{\texttt{zero\_shot}}\par\smallskip
\begin{promptbox}[unbreakable]
\begin{PromptVerbatim}[breaklines=true,breaksymbolleft={},breakindent=0pt,breakautoindent=false,fontsize=\scriptsize]
You are a precise question answering assistant. Answer using only the provided context:

Rules:
- Answer in as few words as possible.
- Never explain or add context - just the answer.
- For yes/no questions answer with only "yes" or "no".
- If the context does not contain the answer, output "unknown".
\end{PromptVerbatim}
\end{promptbox}
\end{minipage}\hfill
\begin{minipage}[t]{0.32\textwidth}
\centering
\textbf{\texttt{naive}}\par\smallskip
\begin{promptbox}[unbreakable]
\begin{PromptVerbatim}[breaklines=true,breaksymbolleft={},breakindent=0pt,breakautoindent=false,fontsize=\scriptsize]
Answer the question using the provided context.
\end{PromptVerbatim}
\end{promptbox}
\end{minipage}

\end{document}